\documentclass{article}

\usepackage[final]{neurips_2023}

\usepackage[utf8]{inputenc} 
\usepackage[T1]{fontenc}    
\usepackage{hyperref}       
\usepackage{url}            
\usepackage{booktabs}       
\usepackage{amsfonts}       
\usepackage{nicefrac}       
\usepackage{microtype}      
\usepackage{xcolor}         

\usepackage{caption}
\usepackage{subcaption}
\usepackage{amsmath}
\usepackage{graphicx}
\usepackage{algorithm}
\usepackage{algorithmic}
\usepackage{amsthm}
\usepackage{amssymb}
\usepackage{thm-restate}

\newtheorem{theorem}{Theorem}[section]

\newtheorem{lemma}[theorem]{Lemma}

\title{Trust Region-Based Safe Distributional \\ Reinforcement Learning for Multiple Constraints}

\author{%
  Dohyeong Kim$^1$, Kyungjae Lee$^2$, and Songhwai Oh$^1$ \\
  $^{1}$Dep. of Electrical and Computer Engineering and ASRI, Seoul National University \\
  $^2$Artificial Intelligence Graduate School, Chung-Ang University \\
  \texttt{dohyeong.kim@rllab.snu.ac.kr},  
  \texttt{kyungjae.lee@ai.cau.ac.kr}, \\
  \texttt{songhwai@snu.ac.kr} \\
}

\begin{document}

\maketitle

\begin{abstract}
In safety-critical robotic tasks, potential failures must be reduced, and multiple constraints must be met, such as avoiding collisions, limiting energy consumption, and maintaining balance.
Thus, applying safe reinforcement learning (RL) in such robotic tasks requires to handle multiple constraints and use risk-averse constraints rather than risk-neutral constraints.
To this end, we propose a trust region-based safe RL algorithm for multiple constraints called a \emph{safe distributional actor-critic} (SDAC).
Our main contributions are as follows: 1) introducing a gradient integration method to manage infeasibility issues in multi-constrained problems, ensuring theoretical convergence, and 2) developing a TD($\lambda$) target distribution to estimate risk-averse constraints with low biases. 
We evaluate SDAC through extensive experiments involving multi- and single-constrained robotic tasks.
While maintaining high scores, SDAC shows 1.93 times fewer steps to satisfy all constraints in multi-constrained tasks and 1.78 times fewer constraint violations in single-constrained tasks compared to safe RL baselines.
Code is available at: \texttt{https://github.com/rllab-snu/Safe-Distributional-Actor-Critic}.
\end{abstract}

\section{Introduction}

Deep reinforcement learning (RL) enables reliable control of complex robots \citep{merel2020humanoid, peng2021amp, rudin2022walk}.
In order to successfully apply RL to real-world systems, it is essential to design a proper reward function which reflects safety guidelines, such as collision avoidance and limited energy consumption, as well as the goal of the given task.
However, finding the reward function that considers all such factors involves a cumbersome and time-consuming process since RL algorithms must be repeatedly performed to verify the results of the designed reward function.
Instead, \emph{safe RL}, which handles safety guidelines as constraints, is an appropriate solution.
A safe RL problem can be formulated using a constrained Markov decision process \citep{altman1999cmdp}, where not only the reward but also cost functions are defined to provide the safety guideline signals.
By defining constraints using expectation or risk measures of the sum of costs, safe RL aims to maximize returns while satisfying the constraints.
Under the safe RL framework, the training process becomes straightforward since there is no need to search for a reward that reflects the safety guidelines.

While various safe RL algorithms have been proposed to deal with the safety guidelines, their applicability to general robotic applications remains limited due to the insufficiency in \textbf{1)} handling multiple constraints and \textbf{2)} minimizing failures, such as robot breakdowns after collisions.
First, many safety-critical applications require multiple constraints, such as maintaining distance from obstacles, limiting operational space, and preventing falls.
Lagrange-based safe RL methods \citep{yang2020wcsac, zhang2022safedist, bai2022zeroviolate}, which convert a safe RL problem into a dual problem and update the policy and Lagrange multipliers, are commonly used to solve these multi-constrained problems.
However, the Lagrangian methods are difficult to guarantee satisfying constraints during training theoretically, and the training process can be unstable due to the multipliers \citep{stooke2020pid}.
To this end, trust region-based methods \citep{yang2020pcpo, kim2022offtrc}, which can ensure to improve returns while satisfying the constraints under tabular settings \citep{achiam2017cpo}, have been proposed as an alternative to stabilize the training process.
Still, trust region-based methods have a critical issue.
Depending on the initial policy settings, there can be an infeasible starting case, meaning that no policy within the trust region satisfies constraints.
To address this issue, we can sequentially select a violated constraint and update the policy to reduce the selected constraint \citep{xu2021crpo}. 
However, this can be inefficient as only one constraint is considered per update.
It will be better to handle multiple constraints at once, but it is a remaining problem to find a policy gradient that reflects several constraints and guarantees to reach a feasible policy set. 

Secondly, as RL settings are inherently stochastic, employing risk-neutral measures like expectation to define constraints can lead to frequent failures. 
Hence, it is crucial to define constraints using risk measures, such as conditional value at risk (CVaR), as they can reduce the potential for massive cost returns by emphasizing tail distributions \citep{yang2020wcsac, kim2022offtrc}.
In safe RL, critics are used to estimate the constraint values.
Especially, to estimate constraints based on risk measures, it is required to use distributional critics \citep{dabney2018quantile}, which can be trained using the distributional Bellman update \citep{bellemare17distributional}.
However, the Bellman update only considers the one-step temporal difference, which can induce a large bias.
The estimation bias makes it difficult for critics to judge the policy, which can lead to the policy becoming overly conservative or risky, as shown in Section \ref{sec:ablation study in main text}.
In particular, when there are multiple constraints, the likelihood of deriving incorrect policy gradients due to estimation errors grows exponentially.
Therefore, there is a need for a method that can train distributional critics with low biases.

In this paper, we propose a trust region-based safe RL algorithm called a \emph{safe distributional actor-critic} (SDAC), designed to effectively manage multiple constraints and estimate risk-averse constraints with low biases.
First, to handle the infeasible starting case by considering all constraints simultaneously, we propose a \emph{gradient integration method} that projects unsafe policies into a feasible policy set by solving a quadratic program (QP) consisting of gradients of all constraints.
It guarantees to obtain a feasible policy within a finite time under mild technical assumptions, and we experimentally show that it can restore the policy more stably than the existing method \citep{xu2021crpo}. 
Furthermore, by updating the policy using the trust region method with the integrated gradient, our approach makes the training process more stable than the Lagrangian method, as demonstrated in Section \ref{sec: locomotion tasks}.
Second, to train critics to estimate constraints with low biases, we propose a \emph{TD($\lambda$) target distribution} which can adjust the bias-variance trade-off. The target distribution is obtained by merging the quantile regression losses \citep{dabney2018quantile} of multi-step distributions and extracting a unified distribution from the loss.
The unified distribution is then projected onto a quantile distribution set in a memory-efficient manner. 
We experimentally show that the target distribution can trade off the bias-variance of the constraint estimations (see Section \ref{sec:ablation study in main text}).

We conduct extensive experiments with multi-constrained locomotion tasks and single-constrained Safety Gym tasks \citep{ray2019safetygym} to evaluate the proposed method. 
In the locomotion tasks, SDAC shows 1.93 times fewer steps to satisfy all constraints than the second-best baselines.
In the Safety Gym tasks, the proposed method shows 1.78 times fewer constraint violations than the second-best methods while achieving high returns when using risk-averse constraints. 
As a result, it is shown that the proposed method can efficiently handle multiple constraints using the gradient integration method and effectively lower the constraint violations using the low-biased distributional critics.

\section{Background}

\textbf{Constrained Markov Decision Processes.}
We formulate the safe RL problem using constrained Markov decision processes (CMDPs) \citep{altman1999cmdp}.
A CMDP is defined as $(S$, $A$, $P$, $R$, $C_{1,..,K}$, $\rho$, $\gamma)$, where $S$ is a state space, $A$ is an action space, $P: S\times A\times S \mapsto [0, 1]$ is a transition model, $R:S\times A\times S \mapsto \mathbb{R}$ is a reward function, $C_{k\in\{1,...,K\}}:S\times A\times S \mapsto \mathbb{R}_{\geq0}$ are cost functions, $\rho:S\mapsto[0, 1]$ is an initial state distribution, and $\gamma \in (0,1)$ is a discount factor.
Given a policy $\pi$ from a stochastic policy set $\Pi$, the discounted state distribution is defined as $d^{\pi}(s):=(1-\gamma)\sum_{t=0}^\infty \gamma^t\mathrm{Pr}(s_t=s|\pi)$, and the return is defined as
$Z_R^{\pi}(s, a) :=\sum_{t=0}^\infty \gamma^t R(s_t, a_t, s_{t+1})$, where $s_0=s, \; a_0 = a, \; a_t\sim\pi(\cdot|s_t)$, and $s_{t+1}\sim P(\cdot|s_t, a_t)$.
Then, the state value and state action value functions are defined as: $V_R^{\pi}(s):=\mathbb{E}\left[Z_R^\pi(s,a)|a \sim \pi(\cdot|s)\right], \; Q_R^{\pi}(s, a):=\mathbb{E}\left[Z_R^\pi(s,a)\right]$.
By substituting the costs for the reward, the cost value functions $V_{C_k}^\pi(s)$ and $Q_{C_k}^\pi(s, a)$ are defined.
In the remainder of the paper, the cost parts will be omitted since they can be retrieved by replacing the reward with the costs.
Then, the safe RL problem is defined as follows with a safety measure $F$:
\begin{equation}
\label{eq:safe RL problem}
\mathrm{maximize}_\pi\;J(\pi) \; \mathbf{s.t.} \; F(Z_{C_k}^{\pi}(s,a)|s\sim\rho, a\sim\pi(\cdot|s)) \leq d_k \; \forall k,    
\end{equation}
where $J(\pi) := \mathbb{E}\left[Z_R^{\pi}(s,a)|s\sim\rho, a\sim\pi(\cdot|s)\right] + \beta\mathbb{E}\left[ \sum_{t=0}^{\infty}\gamma^t H(\pi(\cdot|s_t))|\rho, \pi, P\right]$, $\beta$ is an entropy coefficient, $H$ is the Shannon entropy, and $d_k$ is a threshold of the $k$th constraint.

\textbf{Trust-Region Method With a Mean-Std Constraint.}
\citet{kim2022offtrc} have proposed a trust region-based safe RL method with a risk-averse constraint, called a \emph{mean-std} constraint.
The definition of the mean-std constraint function is as follows:
\begin{equation}
\label{eq:mean-std definition}
F(Z;\alpha):=\mathbb{E}[Z]+(\phi(\Phi^{-1}(\alpha))/\alpha)\cdot\mathrm{Std}[Z],
\end{equation}
where $\alpha \in (0, 1]$ adjusts the risk level of constraints, $\mathrm{Std}[Z]$ is the standard deviation of $Z$, and $\phi$ and $\Phi$ are the probability density function and the cumulative distribution function (CDF) of the standard normal distribution, respectively.
In particular, setting $\alpha=1$ causes the standard deviation part to be zero, so the constraint becomes a risk-neutral constraint.
Also, the mean-std constraint can effectively reduce the potential for massive cost returns, as shown in \citet{yang2020wcsac, kim2022offtrc, kim2022trc}.
In order to calculate the mean-std constraint, it is essential to estimate the standard deviation of the cost return.
To this end, \citet{kim2022offtrc} define the square value functions: 
\begin{equation}
S_{C_k}^{\pi}(s):=\mathbb{E}\left[Z_{C_k}^{\pi}(s,a)^2|a\sim\pi(\cdot|s)\right], \; S_{C_k}^{\pi}(s, a):=\mathbb{E}\left[Z_{C_k}^{\pi}(s,a)^2\right]. 
\end{equation}
Since $\mathrm{Std}[Z]^2 = \mathbb{E}[Z^2] - \mathbb{E}[Z]^2$, the $k$th constraint can be written as follows:
\begin{equation}
\label{eq:mean-std constraint}
\begin{aligned}
&F_k(\pi;\alpha) = J_{C_k}(\pi) + (\phi(\Phi^{-1}(\alpha))/\alpha)\cdot\sqrt{J_{S_k}(\pi) - J_{C_k}(\pi)^2} \leq d_k,    
\end{aligned}
\end{equation}
where $J_{C_k}(\pi) := \mathbb{E}\left[V_{C_k}^{\pi}(s)|s\sim\rho\right]$, $J_{S_k}(\pi) := \mathbb{E}\left[S_{C_k}^{\pi}(s)|s\sim\rho\right]$.
In order to apply the trust region method \citep{schulman2015trpo}, it is necessary to derive surrogate functions for the objective and constraints.
These surrogates can substitute for the objective and constraints within the trust region.
Given a behavioral policy $\mu$ and the current policy $\pi_\mathrm{old}$, we denote the surrogates as $J^{\mu, \pi_\mathrm{old}}(\pi)$ and $F_k^{\mu, \pi_\mathrm{old}}(\pi;\alpha)$.
For the definition and derivation of the surrogates, please refer to Appendix \ref{sec:surrogate} and \citep{kim2022offtrc}.
Using the surrogates, a policy can be updated by solving the following subproblem:
\begin{equation}
\label{eq:safe RL subproblem}
\mathrm{maximize}_{\pi'}\;J^{\mu, \pi}(\pi') \; \mathbf{s.t.} \; D_{KL}(\pi||\pi') \leq \epsilon, F_k^{\mu, \pi}(\pi', \alpha) \leq d_k \; \forall k,    
\end{equation}
where $D_{\mathrm{KL}}(\pi||\pi'):=\mathbb{E}_{s\sim d^{\mu}}\left[D_{\mathrm{KL}}(\pi(\cdot|s)||\pi'(\cdot|s))\right]$, $D_\mathrm{KL}$ is the KL divergence, and $\epsilon$ is a trust region size.
This subproblem can be solved through approximation and a line search (see Appendix \ref{sec:policy update rule}).
However, it is possible that there is no policy satisfying the constraints of (\ref{eq:safe RL subproblem}). 
In order to tackle this issue, the policy must be projected onto a feasible policy set that complies with all constraints, yet there is a lack of such methods. 
In light of this issue, we introduce an efficient feasibility handling method for multi-constrained RL problems.

\textbf{Distributional Quantile Critic.}
\citet{dabney2018quantile} have proposed a method for approximating the random variable $Z^\pi_R$ to follow a quantile distribution.
Given a parametric model, $\theta:S\times A\mapsto \mathbb{R}^M$, $Z_R^\pi$ can be approximated as $Z_{R, \theta}$, called a \emph{distributional quantile critic}.
The probability density function of $Z_{R, \theta}$ is defined as follows:
\begin{equation}
\mathrm{Pr}(Z_{R,\theta}(s,a)=z) := \sum\nolimits_{m=1}^{M}\delta_{\theta_m(s, a)}(z)/M,
\end{equation}
where $M$ is the number of atoms, $\theta_m(s, a)$ is the $m$th atom, $\delta$ is the Dirac function, and $\delta_a(z) := \delta(z-a)$.
The percentile value of the $m$th atom is denoted by $\tau_{m}$ ($\tau_0=0, \tau_i=i/M$).
In distributional RL, the returns are directly estimated to get value functions, and the target distribution can be calculated from the distributional Bellman operator \citep{bellemare17distributional}: $\mathcal{T}^{\pi}Z_R(s, a):\stackrel{D}{=}R(s,a,s')+\gamma Z_R(s',a')$, where $s'\sim P(\cdot|s,a)$ and $a'\sim\pi(\cdot|s')$.
The above one-step distributional operator can be expanded to the $n$-step one: $\mathcal{T}_n^{\pi}Z_R(s_0, a_0):\stackrel{D}{=}\sum_{t=0}^{n-1}\gamma^tR(s_t,a_t,s_{t+1})+\gamma^n Z_R(s_n, a_n)$, where $a_t \sim \pi(\cdot|s_t)$ for $t=1, ..., n$.
Then, the critic can be trained to minimize the following quantile regression loss \citep{dabney2018quantile}:
\begin{equation}
\label{eq:quantile regression loss}
\small
\mathcal{L}(\theta)= \sum_{m=1}^M\underbrace{\mathbb{E}_{(s,a) \sim D} \left[\mathbb{E}_{Z \sim \mathcal{T}^{\pi}Z_{R, \theta}(s,a)}\left[\rho_{\bar{\tau}_m}(Z - \theta_m(s,a))\right] \right]}_{=:\mathcal{L}^{\bar{\tau}_m}_{\mathrm{QR}}(\theta_m)}, \; \mathrm{where} \; \rho_\tau(x) = x\cdot (\tau - \mathbf{1}_\mathrm{x<0}),
\end{equation}
$D$ is a replay buffer, $\bar{\tau}_m:=(\tau_{m-1}+\tau_{m})/2$, and $\mathcal{L}^{\bar{\tau}_m}_{\mathrm{QR}}(\theta_m)$ denotes the quantile regression loss for a single atom.
The distributional quantile critic can be plugged into existing actor-critic algorithms because only the critic modeling part is changed.

\section{Proposed Method}

The proposed method comprises two key components: \textbf{1)} a feasibility handling method required for multi-constrained safe RL problems and \textbf{2)} a target distribution designed to minimize estimation bias. This section sequentially presents these components, followed by a detailed explanation of the proposed method.

\begin{figure}[t]
\centering
\includegraphics[width=0.8\linewidth]{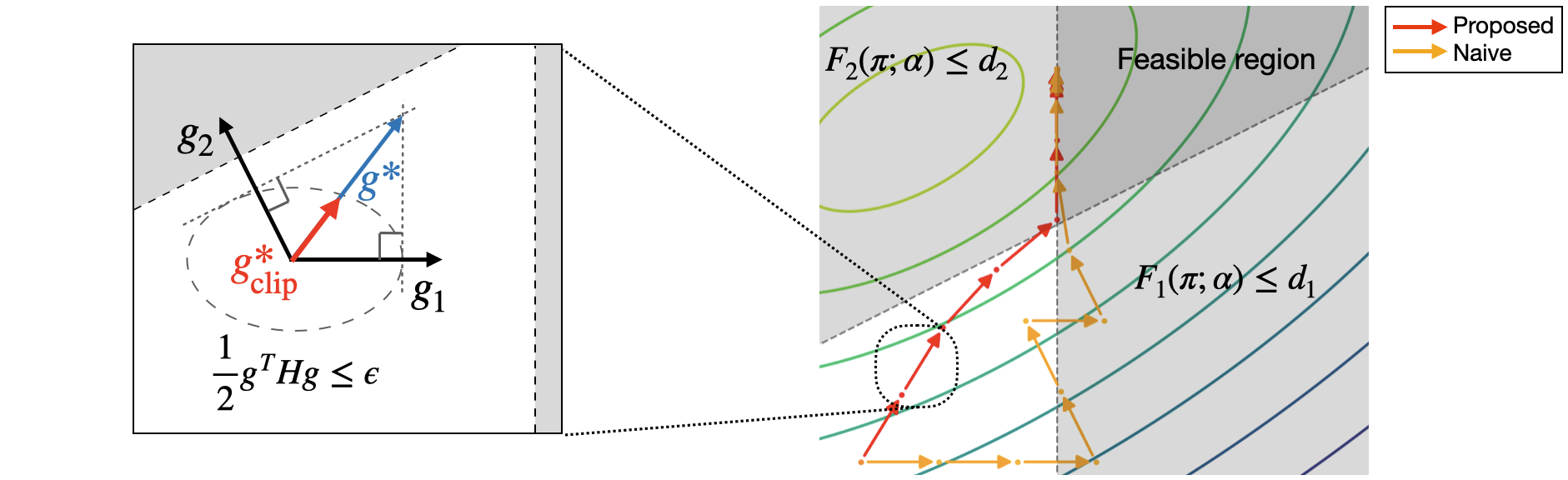}
\caption{
\small
\textbf{Left: Gradient integration.} 
Each constraint is truncated to be tangent to the trust region indicated by the ellipse, and the dashed straight lines show the truncated constraints.
The solution of (\ref{eq: gradient integration}) is indicated in blue, pointing to the nearest point in the intersection of the truncated constraints.
If the solution crosses the trust region, parameters are updated by the clipped direction, shown in red.
\textbf{Right: Optimization paths of the proposed and naive method in a toy example}.
The description is presented in Appendix \ref{sec:toy_example_fig_grad_integ}.
The contour graph represents the objective of the toy example.
The optimization paths exhibit distinct characteristics due to the difference that the naive method reflects only one constraint and the proposed method considers all constraints at once.
}
\label{fig:grad integration}
\end{figure}

\subsection{Feasibility Handling For Multiple Constraints}
\label{sec:feasibility handling}

An optimal safe RL policy can be found by iteratively solving the subproblem (\ref{eq:safe RL subproblem}), but the feasible set of (\ref{eq:safe RL subproblem}) can be empty in the infeasible starting cases. 
To address the feasibility issue in safe RL with multiple constraints, one of the violated constraints can be selected, and the policy is updated to minimize the constraint until the feasible region is not empty \citep{xu2021crpo}, which is called a \emph{naive approach}.
However, it may not be easy to quickly reach the feasible condition if only one constraint at each update step is used to update the policy.
Therefore, we propose a feasibility handling method which reflect all the constraints simultaneously, called a \emph{gradient integration method}.
The main idea is to get a gradient that reduces the value of violated constraints and keeps unviolated constraints.
To find such a gradient, the following quadratic program (QP) can be formulated by linearly approximating the constraints:
\begin{equation}
\label{eq: gradient integration}
\small
\begin{aligned}
&g^* = \underset{g}{\mathrm{argmin}} \; \frac{1}{2}g^THg \;\; \mathbf{s.t.} \; g_k^T g + c_k \leq 0 \; \forall k, \; c_k:=\mathrm{min}(\sqrt{\smash[b]{2\epsilon g_k^T H^{-1}g_k}}, F_k(\pi_{\psi};\alpha) -d_k + \zeta),
\end{aligned}
\end{equation}
where $H$ is the Hessian of KL divergence between the previous policy and the current policy with parameters $\psi$, $g_k$ is the gradient of the $k$th constraint, $c_k$ is a truncated threshold to make the $k$th constraint tangent to the trust region, $\epsilon$ is a trust region size, and $\zeta \in \mathbb{R}_{>0}$ is a slack coefficient.
The reason why we truncate constraints is to make the gradient integration method invariant to the gradient scale.
Otherwise, constraints with larger gradient scales might produce a dominant policy gradient.
Finally, we update the policy parameters using the clipped gradient as follows:
\begin{equation}
\psi^* = \psi + \mathrm{min}(1, \sqrt{\smash[b]{2\epsilon/(g^{*T}Hg^*)}}) g^*.  
\end{equation}
Figure \ref{fig:grad integration} illustrates the proposed gradient integration process.
In summary, the policy is updated by solving (\ref{eq:safe RL subproblem}); if there is no solution to (\ref{eq:safe RL subproblem}), it is updated using the gradient integration method.
Then, the policy can reach the feasibility condition within finite time steps.
\begin{restatable}{theorem}{feasible}
\label{thm:feasible}
Assume that the constraints are differentiable and convex, gradients of the constraints are $L$-Lipschitz continuous, eigenvalues of the Hessian are equal or greater than a positive value $R\in\mathbb{R}_{>0}$, and $\{\psi|F_k(\pi_\psi;\alpha) + \zeta < d_k, \; \forall k\} \neq \emptyset$.
Then, there exists $E \in \mathbb{R}_{>0}$ such that if $0 < \epsilon \leq E$ and a policy is updated by the proposed gradient integration method, all constraints are satisfied within finite time steps.
\end{restatable}
Note that the first two assumptions of Theorem \ref{thm:feasible} are commonly used in multi-task learning \citep{liu2021cagrad, yu2020pcgrad, navon2022nash}, and the assumption on eigenvalues is used in most trust region-based RL methods \citep{schulman2015trpo, kim2022offtrc}, so the assumptions in Theorem \ref{thm:feasible} can be considered reasonable. 
We provide the proof and show the existence of a solution (\ref{eq: gradient integration}) in Appendix \ref{sec:proof_thm2}.
The provided proof shows that the constant $E$ is proportional to $\zeta$.
This means that the trust region size should be set smaller as $\zeta$ decreases.
Also, we further analyze the worst-case time to satisfy all constraints by comparing the gradient integration method and naive approach in Appendix \ref{sec:worst case time}.
In conclusion, if the policy update rule (\ref{eq:safe RL subproblem}) is not feasible, a finite number of applications of the gradient integration method will make the policy feasible. 

\subsection{TD($\lambda$) Target Distribution}
\label{sec:target value distribution}

The mean-std constraints can be estimated using the distributional quantile critics.
Since the estimated constraints obtained from the critics are directly used to update policies in (\ref{eq:safe RL subproblem}), estimating the constraints with low biases is crucial.
In order to reduce the estimation bias of the critics, we propose a target distribution by capturing that the TD($\lambda$) loss, which is obtained by a weighted sum of several losses, and the quantile regression loss with a single distribution are identical.
A recursive method is then introduced so that the target distribution can be obtained practically.
First, the $n$-step targets for the current policy $\pi$ are estimated as follows, after collecting trajectories $(s_t, a_t, s_{t+1}, ...)$ with a behavioral policy $\mu$:
\begin{equation}
\label{eq:sampled n step target}
\begin{aligned}
\hat{Z}^{(n)}_R(s_t, a_t) &:\stackrel{D}{=}R_{t} + \gamma R_{t+1} + \cdots + \gamma^{n-1}R_{t+n-1} + \gamma^n Z_{R,\theta}(s_{t+n}, \hat{a}_{t+n}),
\end{aligned}
\end{equation}
where $R_{t}=R(s_{t},a_{t},s_{t+1})$, and $\hat{a}_{t+n} \sim \pi(\cdot|s_{t+n})$. 
Note that the $n$-step target controls the bias-variance tradeoff using $n$. If $n$ is equal to $1$, the $n$-step target is equivalent to the temporal difference method that has low variance but high bias. On the contrary, if $n$ increases to infinity, it becomes a Monte Carlo estimation that has high variance but low bias.
However, finding proper $n$ is another cumbersome task. To alleviate this issue, TD($\lambda$) \citep{sutton1988learning} method considers the discounted sum of all $n$-step targets. 
Similar to TD($\lambda$), we define the TD($\lambda$) loss for the distributional quantile critic as the discounted sum of all quantile regression losses with $n$-step targets.
Then, the TD($\lambda$) loss for a single atom is approximated using importance sampling of the sampled $n$-step targets in (\ref{eq:sampled n step target}) as:
\begin{equation}
\label{eq:total quantile loss}
\small
\begin{aligned}
&\mathcal{L}^{\bar{\tau}_m}_{\mathrm{QR}}(\theta_m) = (1 - \lambda)\sum_{i=0}^\infty\lambda^i \mathbb{E}_{(s_t, a_t) \sim D} \left[ \mathbb{E}_{Z \sim\mathcal{T}_{i+1}^{\pi}Z_{R,\theta}(s_t, a_t)} \left[ \rho_{\bar{\tau}_m}(Z - \theta_m(s_t, a_t))\right] \right] \\
& = (1 - \lambda)\sum_{i=0}^\infty\lambda^i \mathbb{E}_{(s_t, a_t) \sim D} \left[ \mathbb{E}_{Z \sim \mathcal{T}_{i+1}^{\mu}Z_{R,\theta}(s_t, a_t)} \left[ \prod_{j=t+1}^{t+i}\frac{\pi(a_{j}|s_{j})}{\mu(a_{j}|s_{j})} \rho_{\bar{\tau}_m}(Z - \theta_m (s_t, a_t))\right] \right] \\
&\approx (1-\lambda)\sum_{i=0}^\infty\lambda^i \mathbb{E}_{(s_t, a_t) \sim D} \left[ \prod_{j=t+1}^{t+i}\frac{\pi(a_{j}|s_{j})}{\mu(a_{j}|s_{j})}\sum_{m=1}^{M} \frac{1}{M} \rho_{\bar{\tau}_m}(\hat{Z}^{(i+1)}_{R,m}(s_t, a_t) - \theta_m(s_t, a_t)) \right],
\end{aligned}
\end{equation}
where $\lambda$ is a trace-decay value, and $\hat{Z}^{(i)}_{R,m}$ is the $m$th atom of $\hat{Z}^{(i)}_{R}$.
Since $\hat{Z}_{R}^{(i)}(s_t,a_t) \stackrel{D}{=} R_{t} + \gamma\hat{Z}_{R}^{(i-1)}(s_{t+1}, a_{t+1})$ is satisfied, (\ref{eq:total quantile loss}) is the same as the quantile regression loss with the following single distribution $\hat{Z}_t^{\mathrm{tot}}$, called a \emph{TD($\lambda$) target distribution}:
\begin{equation}
\label{eq:target distribution}
\small
\begin{aligned}
&\mathrm{Pr}(\hat{Z}_t^{\mathrm{tot}} = z) := \frac{1}{\mathcal{N}_t}(1-\lambda)\sum_{i=0}^\infty\lambda^i\prod_{j=t+1}^{t+i}\frac{\pi(a_{j}|s_{j})}{\mu(a_{j}|s_{j})}\sum_{m=1}^{M} \frac{1}{M} \delta_{\hat{Z}^{(i+1)}_{R,m}(s_t, a_t)}(z) \\
& = \frac{1-\lambda}{\mathcal{N}_t}\bigg(\sum_{m=1}^{M}\frac{1}{M}\delta_{\hat{Z}^{(1)}_{t,m}}(z) + \lambda\frac{\pi(a_{t+1}|s_{t+1})}{\mu(a_{t+1}|s_{t+1})} \sum_{i=0}^\infty\lambda^i\prod_{j=t+2}^{t+1+i}\frac{\pi(a_{j}|s_{j})}{\mu(a_{j}|s_{j})} \sum_{m=1}^{M} \frac{1}{M} \delta_{\hat{Z}^{(i+2)}_{R,m}(s_t, a_t)}(z) \bigg) \\
&=\frac{1-\lambda}{\mathcal{N}_t}\underbrace{\sum_{m=1}^{M}\frac{1}{M}\delta_{\hat{Z}^{(1)}_{t,m}}(z)}_{\text{(a)}} + \lambda\frac{\mathcal{N}_{t+1}}{\mathcal{N}_t}\frac{\pi(a_{t+1}|s_{t+1})}{\mu(a_{t+1}|s_{t+1})} \underbrace{\mathrm{Pr}(R_{t}+\gamma\hat{Z}^{\mathrm{tot}}_{t+1}=z)}_{\text{(b)}},
\end{aligned}
\end{equation}
where $\mathcal{N}_t$ is a normalization factor.
If the target for time step $t+1$ is obtained, the target distribution for time step $t$ becomes the weighted sum of \textbf{(a)} the current one-step TD target and \textbf{(b)} the shifted previous target distribution, so it can be obtained recursively, as shown in (\ref{eq:target distribution}).
Since the definition requires infinite sums, the recursive way is more practical for computing the target.
Nevertheless, to obtain the target in that recursive way, we need to store all quantile positions and weights for all time steps, which is not memory-efficient.
Therefore, we propose to project the target distribution into a quantile distribution with a specific number of atoms, $M'$ (we set $M' > M$ to reduce information loss).
The overall process to get the TD$(\lambda)$ target distribution is illustrated in Figure \ref{fig:target distribution}, and the pseudocode is given in Appendix \ref{sec:pseudocode_target}.
Furthermore, we can show that a distribution trained with the proposed target converges to the distribution of $Z_R^{\pi}$.
\begin{restatable}{theorem}{converge}
\label{thm:target}
Let define a distributional operator $\mathcal{T}_\lambda^{\mu,\pi}$, whose probability density function is: 
\begin{equation}
\label{eq:distributional_operator}
\small
\begin{aligned}
&\mathrm{Pr}(\mathcal{T}_\lambda^{\mu,\pi}Z(s,a)\!=\!z) \propto \\
&\;\; \sum_{i=0}^\infty \mathbb{E}_\mu \! \left[\lambda^i\prod_{j=1}^i\frac{\pi(a_j|s_j)}{\mu(a_j|s_j)} \underset{a'\sim \pi(\cdot|s_{i+1})}{\mathbb{E}}\!\left[\mathrm{Pr}\left(\sum_{t=0}^i\gamma^t R_t \!+ \! \gamma^{i+1}Z(s_{i+1}, a')\!=\!z \right)\right] \!\! \bigg|s_0=s, a_0=a\right].
\end{aligned}
\end{equation}
Then, a sequence, $Z_{k+1}(s,a)=\mathcal{T}_\lambda^{\mu,\pi}Z_k(s, a) \; \forall(s, a)$, converges to $Z_R^\pi$.
\end{restatable}
The TD$(\lambda)$ target is a quantile distribution version of the distributional operator $\mathcal{T}_\lambda^{\mu,\pi}$ in Theorem \ref{thm:target}. 
Consequently, a distribution updated by minimizing the quantile regression loss with the TD$(\lambda)$ target converges to the distribution of $Z_R^\pi$ if the number of atoms is infinite, according to Theorem \ref{thm:target}.
The proof of Theorem \ref{thm:target} is provided in Appendix \ref{sec:proof_contraction}.
After calculating the target distribution for all time steps, the critic can be trained to reduce the quantile regression loss with the target distribution.
To provide more insight, we experiment with a toy example in Appendix \ref{sec:toy example of target}, and the results show that the proposed target distribution can trade off bias and variance through the trace-decay $\lambda$.

\begin{figure*}[t]
\begin{center}
\includegraphics[width=0.9\linewidth]{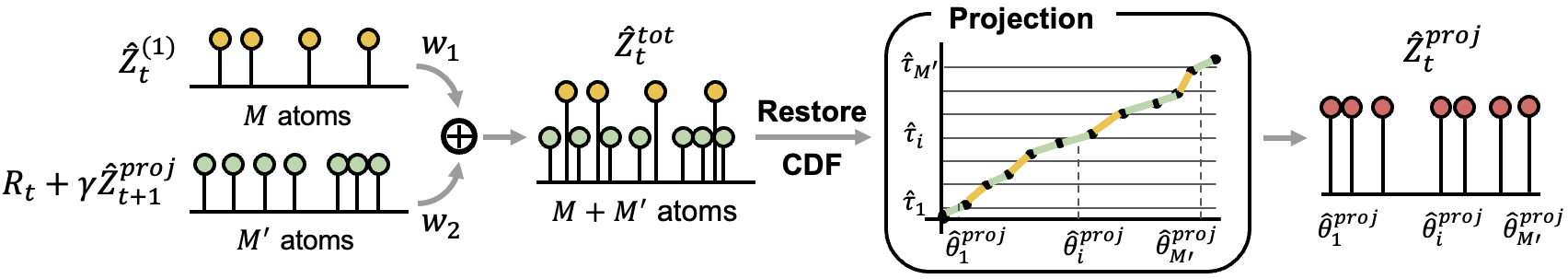}
\end{center}
\caption{
\small
\textbf{Procedure for constructing the target distribution.}
First, multiply the target at $t+1$ step by $\gamma$ and add $R_t$. Next, weight-combine the shifted previous target and one-step target at $t$ step and restore the CDF of the combined target.
The CDF can be restored by sorting the positions of the atoms and then accumulating the weights at each atom position.
Finally, the projected target can be obtained by finding the positions of the atoms corresponding to $M'$ quantiles in the CDF.
Using the projected target, the target at $t-1$ step can be found recursively.}
\label{fig:target distribution}
\end{figure*}

\subsection{Safe Distributional Actor-Critic}
\label{sec:TRSAC}

Finally, we describe the proposed method, safe distributional actor-critic (SDAC).
After collecting trajectories, the policy is updated by solving (\ref{eq:safe RL subproblem}), which can be solved through a line search (for more detail, see Appendix \ref{sec:policy update rule}).
The cost value and the cost square value functions in (\ref{eq:mean-std constraint}) can be obtained using the distributional critics as follows:
\begin{equation}
\begin{aligned}
Q_C^\pi(s, a) &= \int_{-\infty}^{\infty} z \mathrm{Pr}(Z^\pi_C(s,a)=z)dz \approx \frac{1}{M}\sum_{m=1}^{M}\theta_m(s, a), \\
S_C^\pi(s, a) &= \int_{-\infty}^{\infty} z^2 \mathrm{Pr}(Z^\pi_C(s,a)=z)dz \approx \frac{1}{M}\sum_{m=1}^{M}\theta_m(s, a)^2.
\end{aligned}
\end{equation}
If a solution of (\ref{eq:safe RL subproblem}) does not exist, the policy is projected into a feasible region through the proposed gradient integration method.
The critics can also be updated by the regression loss (\ref{eq:quantile regression loss}) between the target distribution obtained from (\ref{eq:target distribution}).
The proposed method is summarized in Algorithm \ref{algo:proposed method}.

\begin{algorithm}[t]
\caption{Safe Distributional Actor-Critic}
\label{algo:proposed method}
\begin{algorithmic}
\STATE {\bfseries Input:} Policy network $\pi_{\psi}$, reward and cost critic networks $Z_{R, \theta}^{\pi}$, $Z_{C_k, \theta}^{\pi}$, and replay buffer $\mathcal{D}$.
\STATE Initialize network parameters $\psi, \theta$, and replay buffer $\mathcal{D}$.
\FOR{epochs$\;=1$ {\bfseries to} $E$}
    \FOR{$t=1$ {\bfseries to} $T$}
        \STATE Sample $a_t\sim\pi_{\psi}(\cdot|s_t)$ and get $s_{t+1}, r_t=R(s_t, a_t, s_{t+1}),$ and $c_{k,t}=C_k(s_t, a_t, s_{t+1}) \; \forall k$.
        \STATE Store $(s_t, a_t, \pi_{\psi}(a_t|s_t), r_t, c_{\{1,...,K\},t}, s_{t+1})$ in $\mathcal{D}$.
    \ENDFOR
    \STATE Calculate the TD($\lambda$) target distribution (Section \ref{sec:target value distribution}) using $\mathcal{D}$ and update the critics to minimize the quantile loss defined in (\ref{eq:quantile regression loss}).
    \STATE Calculate the surrogates for the objective and constraints defined in (\ref{eq:original surrogate}) using $\mathcal{D}$.
    \STATE Update the policy by solving (\ref{eq:safe RL subproblem}), but if (\ref{eq:safe RL subproblem}) has no solution, take a recovery step (Section \ref{sec:feasibility handling}).
\ENDFOR
\end{algorithmic}
\end{algorithm}

\section{Related Work}

\textbf{Safe Reinforcement Learning.}
\citet{garcia2015comprehensive} and \citet{gu2022review} have researched and categorized safe RL methodologies from various perspectives. 
In this paper, we introduce safe RL methods depending on how to update policies to reflect safety constraints.
First, trust region-based safe RL methods \citep{achiam2017cpo, yang2020pcpo, kim2022offtrc} find policy update directions by approximating the safe RL problem as a linear-quadratic constrained linear program and update policies through a line search.
\citet{yang2020pcpo} also employ projection to meet a constraint; however, their method is limited to a single constraint and does not show to satisfy the constraint for the infeasible starting case.
Second, Lagrangian-based methods \citep{stooke2020pid, yang2020wcsac, liu2020ipo} convert the safe RL problem to a dual problem and update the policy and dual variables simultaneously.
Last, expectation-maximization (EM) based methods \citep{liu2022cvpo, zhang2022cdmpo} find non-parametric policy distributions by solving the safe RL problem in E-steps and fit parametric policies to the found non-parametric distributions in M-steps.
Also, there are other ways to reflect safety other than policy updates.
\citet{qin2021density, lee2022coptidice} find optimal state or state-action distributions that satisfy constraints, and \citet{bharadhwaj2021csc, thananjeyan2021rrl} reflect safety during exploration by executing only safe action candidates.
In the experiments, only the safe RL methods of the policy update approach are compared with the proposed method.

\textbf{Distributional TD($\lambda$).}
TD($\lambda$) \citep{precup2000tdtrace} can be extended to the distributional critic to trade off bias-variance.
\citet{gruslys2018reactor} have proposed a method to obtain target distributions by mixing $n$-step distributions, but the method is applicable only in discrete action spaces.
\citet{nam2021gmac} have proposed a method to obtain target distributions using sampling to apply to continuous action spaces, but this is only for on-policy settings.
A method proposed by \citet{tang2022disttd} updates the critics using newly defined distributional TD errors rather than target distributions. This method is applicable for off-policy settings but has the disadvantage that memory usage increases linearly with the number of TD error steps.
In contrast to these methods, the proposed method is memory-efficient and applicable for continuous action spaces under off-policy settings.

\textbf{Gradient Integration.}
The proposed feasibility handling method utilizes a gradient integration method, which is widely used in multi-task learning (MTL).
The gradient integration method finds a single gradient to improve all tasks by using gradients of all tasks.
\citet{yu2020pcgrad} have proposed a projection-based gradient integration method, which is guaranteed to converge Pareto-stationary sets.
A method proposed by \citet{liu2021cagrad} can reflect user preference, and \citet{navon2022nash} proposed a gradient-scale invariant method to prevent the training process from being biased by a few tasks.
The proposed method can be viewed as a mixture of projection and scale-invariant methods as gradients are clipped and projected onto a trust region.

\section{Experiments}
\label{sec:experiments}

We evaluate the safety performance of the proposed method in single- and multi-constrained robotic tasks.
For single constraints, the agent performs four tasks provided by Safety Gym \citep{ray2019safetygym}, and for multi-constraints, it performs bipedal and quadrupedal locomotion tasks.

\subsection{Safety Gym}

\textbf{Tasks.}
We employ two robots, point and car, to perform goal and button tasks in the Safety Gym. 
The goal task is to control a robot toward a randomly spawned goal without passing through hazard regions. 
The button task is to click a randomly designated button using a robot, where not only hazard regions but also dynamic obstacles exist.
Agents get a cost when touching undesignated buttons and obstacles or entering hazard regions. 
There is only one constraint for the Safety Gym tasks, and it is defined using (\ref{eq:mean-std constraint}) with the sum of costs.
Constraint violations (CVs) are counted when the cost sum exceeds the threshold. 
For more details, see Appendix \ref{sec:experimental settings}.

\textbf{Baselines.}
Safe RL methods based on various types of policy updates are used as baselines.
For the trust region-based method, we use constrained policy optimization (CPO) \citep{achiam2017cpo} and off-policy trust-region CVaR (OffTRC) \citep{kim2022offtrc}, which extend the CPO to an off-policy and mean-std constrained version.
For the Lagrangian-based method, distributional worst-case soft actor-critic (WCSAC) \citep{yang2022safety} is used, and constrained variational policy optimization (CVPO) \citep{liu2022cvpo} based on the EM method is used.
Specifically, WCSAC, OffTRC, and the proposed method, SDAC, use the risk-averse constraints, so we experiment with those for $\alpha=0.25$ and $1.0$ (when $\alpha=1.0$, the constraint is identical to the risk-neutral constraint).

\begin{figure*}[t]
\centering
\begin{subfigure}[b]{1.0\textwidth}
    \centering
    \includegraphics[width=1.0\textwidth]{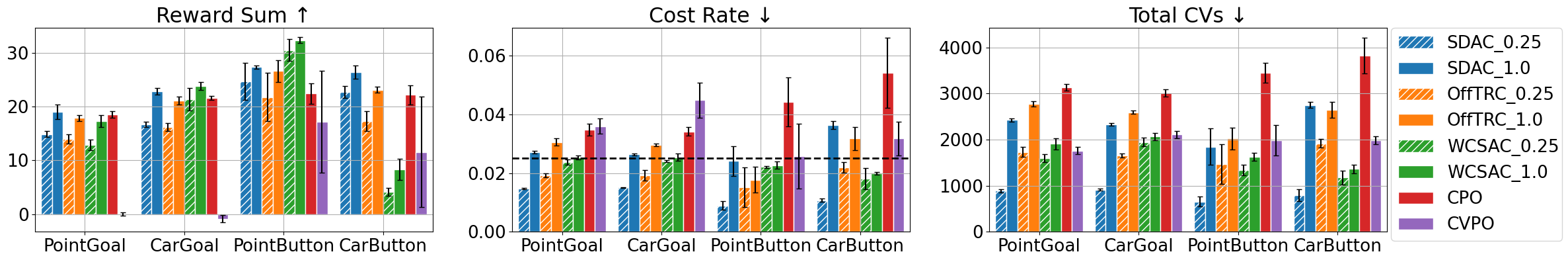}
    \vspace{-15pt}
    \caption{Results of the final reward sum, cost rate, and total number of CVs. The number after the algorithm name in the legend indicates $\alpha$ used for the risk-averse constraint.}
    \label{sfig:safetygym result}
\end{subfigure}
\hfill
\begin{subfigure}[b]{1.0\textwidth}
    \centering
    \includegraphics[width=\textwidth]{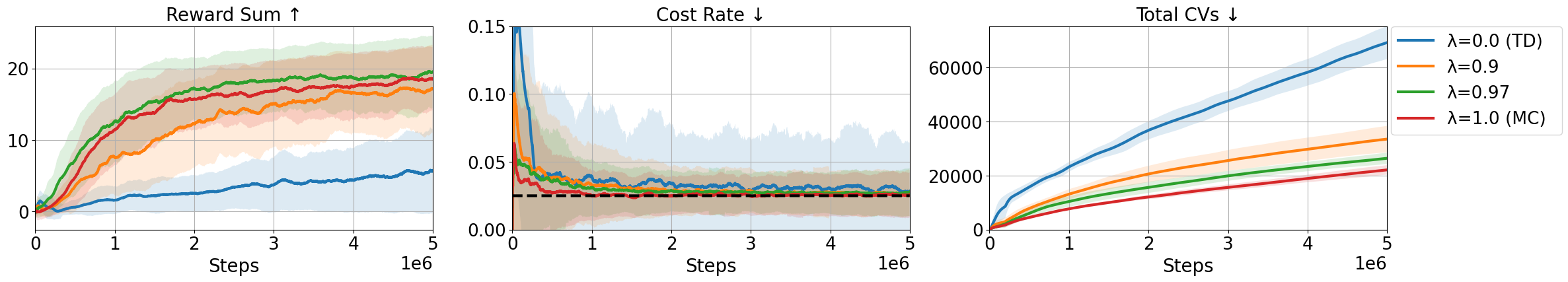}
    \vspace{-15pt}
    \caption{Training curves of the point goal task according to the trace-decay $\lambda$. The solid line represents the average value, and the shaded area shows half of the std value.}
    \label{sfig:ablation on trace decay}
\end{subfigure}
\vspace{-15pt}
\caption{
\small
\textbf{Safety Gym task results.}
The cost rates show the cost sums divided by the episode length, and the dashed black lines indicate the threshold of the constraint.
All methods are trained with five random seeds.
}
\label{fig:safety gym training results}
\vspace{-10pt}
\end{figure*}

\textbf{Results.}
The graph of the final reward sum, cost rate, and the total number of CVs are shown in Figure \ref{sfig:safetygym result}, and the training curves are provided in Appendix \ref{sec:experimental results on safety gym}.
We can interpret the results as good if the reward sum is high and the cost rate and total CVs are low.
SDAC with $\alpha=0.25$, risk-averse constraint situations, satisfies the constraints in all tasks and shows an average of 1.78 times fewer total CVs than the second-best algorithm. 
Nevertheless, since the reward sums are also in the middle or upper ranks, its safety performance is of high quality.
SDAC with $\alpha=1.0$, risk-neutral constraint situations, shows that the cost rates are almost the same as the thresholds except for the car button. In the case of the car button, the constraint is not satisfied, but by setting $\alpha=0.25$, SDAC can achieve the lowest total CVs and the highest reward sum compared to the other methods.
As for the reward sum, SDAC is the highest in the point goal and car button, and WCSAC is the highest in the rest.
However, WCSAC seems to lose the risk-averse properties seeing that the cost rates do not change significantly according to $\alpha$.
This is because WCSAC does not define constraints as risk measures of cost returns but as expectations of risk measures \citep{yang2022safety}.
OffTRC has lower safety performance than SDAC in most cases because, unlike SDAC, it does not use distributional critics.
Finally, CVPO and CPO are on-policy methods, so they are less efficient than the other methods.

\subsection{Locomotion Tasks}
\label{sec: locomotion tasks}

\textbf{Tasks.}
The locomotion tasks are to train robots to follow $xy$-directional linear and $z$-directional angular velocity commands. 
Mini-Cheetah from MIT \citep{katz2019minicheetah} and Laikago from Unitree \citep{laikago2018} are used for quadrupedal robots, and Cassie from Agility Robotics \citep{xie2018cassie} is used for a bipedal robot. 
In order to perform the locomotion tasks, robots should keep balancing, standing, and stamping their feet so that they can move in any direction.
Therefore, we define three constraints.
The first is to keep the balance so that the body angle does not deviate from zero, and the second is to keep the height of CoM above a threshold.
The third is to match the current foot contact state with a predefined foot contact timing.
The reward is defined as the negative $l^2$-norm of the difference between the command and the current velocity.
CVs are counted when the sum of at least one cost rate exceeds the threshold.
For more details, see Appendix \ref{sec:experimental settings}.

\textbf{Baselines.}
The baseline methods are identical to the Safety Gym tasks, and CVPO is excluded because it is technically challenging to scale to multiple constraint settings.
We set $\alpha$ to $1$ for the risk-averse constrained methods (OffTRC, WCSAC, and SDAC) to focus on measuring multi-constraint handling performance.

\begin{figure*}[t]
\centering
\begin{subfigure}[b]{1.0\textwidth}
    \centering
    \includegraphics[width=\textwidth]{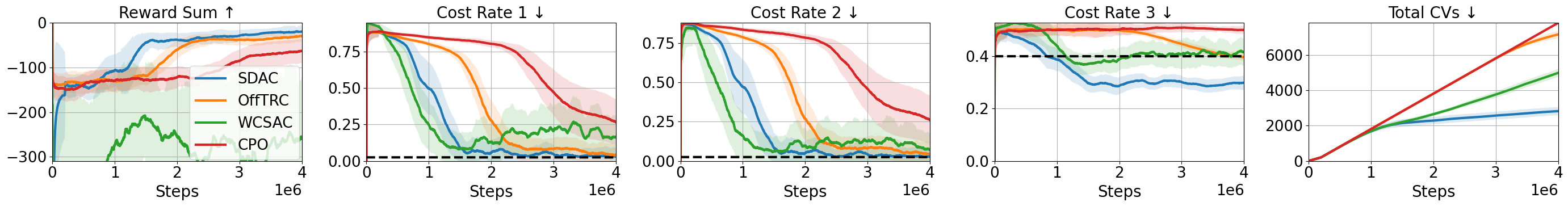}
    \vspace{-15pt}
    \caption{Training curves of the Cassie task.}
    \label{sfig:cassie}
\end{subfigure}
\hfill
\begin{subfigure}[b]{1.0\textwidth}
    \centering
    \includegraphics[width=\textwidth]{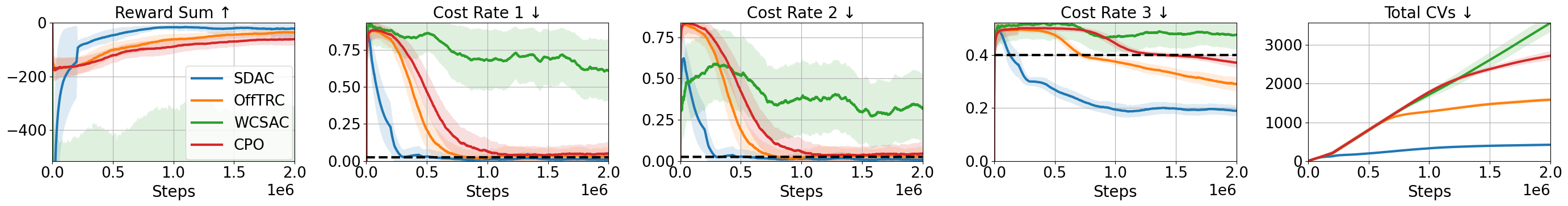}
    \vspace{-15pt}
    \caption{Training curves of the Laikago task.}
    \label{sfig:laikago}
\end{subfigure}
\hfill
\begin{subfigure}[b]{1.0\textwidth}
    \centering
    \includegraphics[width=\textwidth]{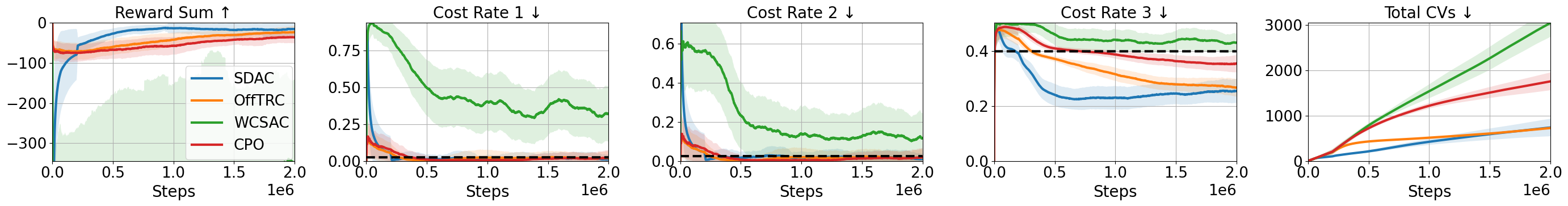}
    \vspace{-15pt}
    \caption{Training curves of the Mini-Cheetah task.}
    \label{sfig:cheetah}
\end{subfigure}
\hfill
\begin{subfigure}[b]{1.0\textwidth}
    \centering
    \includegraphics[width=\textwidth]{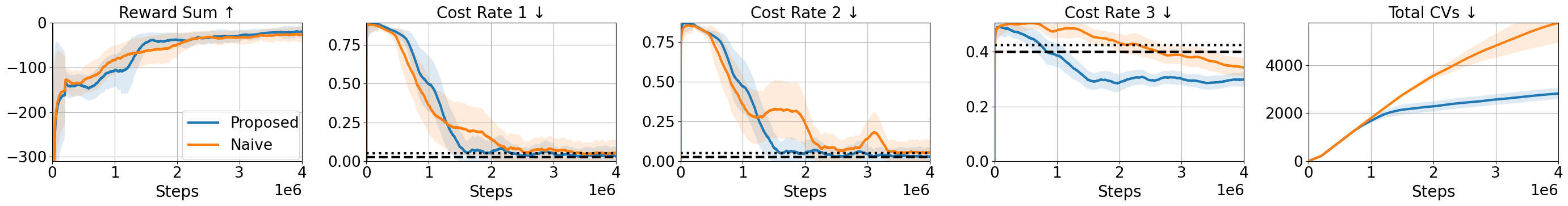}
    \vspace{-15pt}
    \caption{Training curves of the naive and proposed methods for the Cassie task.}
    \label{sfig:ablation on feasibility handling}
\end{subfigure}
\vspace{-15pt}
\caption{
\small
\textbf{Locomotion task results.}
The black dashed lines indicate the thresholds, and the dotted lines in (d) represent the thresholds + $0.025$.
The shaded area represents half of the standard deviation, and all methods are trained with five random seeds.}
\label{fig: locomotion results}
\vspace{-10pt}
\end{figure*}

\textbf{Results.}
Figure \ref{fig: locomotion results}.(a-c) presents the training curves.
SDAC shows the highest reward sums and the lowest total CVs in all tasks. 
In particular, the number of steps required to satisfy all constraints is 1.93 times fewer than the second-best algorithm on average. 
Trust region methods (OffTRC, CPO) stably satisfy constraints, but they are not efficient since they handle constraints by the naive approach.
WCSAC, a Lagrangian method, fails to keep the constraints and shows the lowest reward sums. This is because the Lagrange multipliers can hinder the training stability due to the concurrent update with policy \citep{stooke2020pid}.

\subsection{Ablation Study}
\label{sec:ablation study in main text}
We conduct ablation studies to show whether the proposed target distribution lowers the estimation bias and whether the proposed gradient integration quickly converges to the feasibility condition. 
In Figure \ref{sfig:ablation on trace decay}, the number of CVs is reduced as $\lambda$ increases, which means that the bias of constraint estimation decreases. However, the score also decreases due to large variance, showing that $\lambda$ can adjust the bias-variance tradeoff.
In Figure \ref{sfig:ablation on feasibility handling}, the proposed gradient integration method is compared with the naive approach, which minimizes the constraints in order from the first to the third constraint, as described in Section \ref{sec:feasibility handling}.
The proposed method reaches the feasibility condition faster than the naive approach and shows stable training curves because it reflects all constraints concurrently.
Additionally, we analyze the distributional critics in Appendix \ref{sec: ablation on components of SDAC} and the hyperparameters, such as the trust region size, in Appendix \ref{sec:trace ablation on hyperparameters}.
Furthermore, we analyze the sensitivity of traditional RL algorithms to reward configuration in Appendix \ref{sec:comparison_rl_algo}, emphasizing the advantage of safe RL that does not require reward tuning.

\section{Limitation}

A limitation of the proposed method is that the computational complexity of the gradient integration is proportional to the square of the number of constraints, whose qualitative analysis is presented in Appendix \ref{sec: qualitative complexity}. 
Also, we conducted quantitative analyses in Appendix \ref{sec: quantitative complexity} by measuring wall clock training time. 
In the mini-cheetah task which has three constraints, the training time of SDAC is the third fastest among the four safe RL algorithms. 
Gradient integration is not applied when the policy satisfies constraints, so it may not constitute a significant proportion of training time. 
However, its influence can be dominant as the number of constraints increases.
In order to resolve this limitation, the calculation can be speeded up by stochastically selecting a subset of constraints \citep{liu2021cagrad}, or by reducing the frequency of policy updates \citep{navon2022nash}. 
The other limitation is that the mean-std defined in (\ref{eq:mean-std definition}) is not a coherent risk measure.
As a result, mean-std constraints can be served as reducing uncertainty rather than risk, although we experimentally showed that constraint violations are efficiently reduced.
To resolve this, we can use the CVaR constraint, which can be estimated using an auxiliary variable, as done by \citet{chow2017risk}.
However, this solution can destabilize the training process due to the auxiliary variable, as observed in experiments of \citet{kim2022trc}.
Hence, a stabilization technique should be developed to employ the CVaR constraint.

\section{Conclusion}

We have presented the trust region-based safe distributional RL method, called \emph{SDAC}.
Through the locomotion tasks, it is verified that the proposed method efficiently satisfies multiple constraints using the gradient integration.
Moreover, constraints can be stably satisfied in various tasks due to the low-biased distributional critics trained using the proposed target distributions. 
In addition, the proposed method is analyzed from multiple perspectives through various ablation studies.
However, to compensate for the computational complexity, future work plans to devise efficient methods when dealing with large numbers of constraints.

\begin{ack}
This work was partly supported by Institute of Information \& Communications Technology Planning \& Evaluation (IITP) grant funded by the Korea government (MSIT) (No. 2019-0-01190, [SW Star Lab] Robot Learning: Efficient, Safe, and Socially-Acceptable Machine Learning, 34\%), Basic Science Research Program through the National Research Foundation of Korea (NRF) funded by the Ministry of Science and ICT (NRF-2022R1A2C2008239, General-Purpose Deep Reinforcement Learning Using Metaverse for Real World Applications, 33\%), and Institute of Information \& Communications Technology Planning \& Evaluation (IITP) grant funded by the Korea government (MSIT) (No. 2021-0-01341, AI Graduate School Program, CAU, 33\%).
\end{ack}

\bibliographystyle{abbrvnat}
\bibliography{main}

\begin{thebibliography}{51}
\providecommand{\natexlab}[1]{#1}
\providecommand{\url}[1]{\texttt{#1}}
\expandafter\ifx\csname urlstyle\endcsname\relax
  \providecommand{\doi}[1]{doi: #1}\else
  \providecommand{\doi}{doi: \begingroup \urlstyle{rm}\Url}\fi

\bibitem[Achiam et~al.(2017)Achiam, Held, Tamar, and Abbeel]{achiam2017cpo}
J.~Achiam, D.~Held, A.~Tamar, and P.~Abbeel.
\newblock Constrained policy optimization.
\newblock In \emph{Proceedings of International Conference on Machine Learning}, pages 22--31, 2017.

\bibitem[Agarwal et~al.(2021)Agarwal, Kakade, Lee, and Mahajan]{agarwal2021theory}
A.~Agarwal, S.~M. Kakade, J.~D. Lee, and G.~Mahajan.
\newblock On the theory of policy gradient methods: Optimality, approximation, and distribution shift.
\newblock \emph{The Journal of Machine Learning Research}, 22\penalty0 (1):\penalty0 4431--4506, 2021.

\bibitem[Altman(1999)]{altman1999cmdp}
E.~Altman.
\newblock \emph{Constrained Markov decision processes}, volume~7.
\newblock CRC Press, 1999.

\bibitem[Bai et~al.(2022)Bai, Bedi, Agarwal, Koppel, and Aggarwal]{bai2022zeroviolate}
Q.~Bai, A.~S. Bedi, M.~Agarwal, A.~Koppel, and V.~Aggarwal.
\newblock Achieving zero constraint violation for constrained reinforcement learning via primal-dual approach.
\newblock \emph{Proceedings of the AAAI Conference on Artificial Intelligence}, 36\penalty0 (4), 2022.

\bibitem[Bellemare et~al.(2017)Bellemare, Dabney, and Munos]{bellemare17distributional}
M.~G. Bellemare, W.~Dabney, and R.~Munos.
\newblock A distributional perspective on reinforcement learning.
\newblock In \emph{Proceedings of International Conference on Machine Learning}, pages 449--458, 2017.

\bibitem[Bellemare et~al.(2023)Bellemare, Dabney, and Rowland]{bdr2022}
M.~G. Bellemare, W.~Dabney, and M.~Rowland.
\newblock \emph{Distributional Reinforcement Learning}.
\newblock MIT Press, 2023.
\newblock \url{http://www.distributional-rl.org}.

\bibitem[Bharadhwaj et~al.(2021)Bharadhwaj, Kumar, Rhinehart, Levine, Shkurti, and Garg]{bharadhwaj2021csc}
H.~Bharadhwaj, A.~Kumar, N.~Rhinehart, S.~Levine, F.~Shkurti, and A.~Garg.
\newblock Conservative safety critics for exploration.
\newblock In \emph{Proceedings of International Conference on Learning Representations}, 2021.

\bibitem[Biewald(2020)]{wandb}
L.~Biewald.
\newblock Experiment tracking with weights and biases, 2020.
\newblock URL \url{https://www.wandb.com/}.
\newblock Software available from wandb.com.

\bibitem[Chow et~al.(2017)Chow, Ghavamzadeh, Janson, and Pavone]{chow2017risk}
Y.~Chow, M.~Ghavamzadeh, L.~Janson, and M.~Pavone.
\newblock Risk-constrained reinforcement learning with percentile risk criteria.
\newblock \emph{Journal of Machine Learning Research}, 18\penalty0 (1):\penalty0 6070--6120, 2017.

\bibitem[Dabney et~al.(2018{\natexlab{a}})Dabney, Ostrovski, Silver, and Munos]{dabney2018implicit}
W.~Dabney, G.~Ostrovski, D.~Silver, and R.~Munos.
\newblock Implicit quantile networks for distributional reinforcement learning.
\newblock In \emph{Proceedings of International conference on machine learning}, pages 1096--1105, 2018{\natexlab{a}}.

\bibitem[Dabney et~al.(2018{\natexlab{b}})Dabney, Rowland, Bellemare, and Munos]{dabney2018quantile}
W.~Dabney, M.~Rowland, M.~Bellemare, and R.~Munos.
\newblock Distributional reinforcement learning with quantile regression.
\newblock \emph{Proceedings of the AAAI Conference on Artificial Intelligence}, 32\penalty0 (1), 2018{\natexlab{b}}.

\bibitem[Dennis and Schnabel(1996)]{dennis1996numerical}
J.~E. Dennis and R.~B. Schnabel.
\newblock \emph{Numerical Methods for Unconstrained Optimization and Nonlinear Equations}.
\newblock Society for Industrial and Applied Mathematics, 1996.

\bibitem[Garc{\i}a and Fern{\'a}ndez(2015)]{garcia2015comprehensive}
J.~Garc{\i}a and F.~Fern{\'a}ndez.
\newblock A comprehensive survey on safe reinforcement learning.
\newblock \emph{Journal of Machine Learning Research}, 16\penalty0 (1):\penalty0 1437--1480, 2015.

\bibitem[Gruslys et~al.(2018)Gruslys, Dabney, Azar, Piot, Bellemare, and Munos]{gruslys2018reactor}
A.~Gruslys, W.~Dabney, M.~G. Azar, B.~Piot, M.~Bellemare, and R.~Munos.
\newblock The reactor: A fast and sample-efficient actor-critic agent for reinforcement learning.
\newblock In \emph{Proceedings of International Conference on Learning Representations}, 2018.

\bibitem[Gu et~al.(2022)Gu, Yang, Du, Chen, Walter, Wang, Yang, and Knoll]{gu2022review}
S.~Gu, L.~Yang, Y.~Du, G.~Chen, F.~Walter, J.~Wang, Y.~Yang, and A.~Knoll.
\newblock A review of safe reinforcement learning: Methods, theory and applications.
\newblock \emph{arXiv preprint arXiv:2205.10330}, 2022.

\bibitem[Haarnoja et~al.(2018)Haarnoja, Zhou, Abbeel, and Levine]{haarnoja2018sac}
T.~Haarnoja, A.~Zhou, P.~Abbeel, and S.~Levine.
\newblock Soft actor-critic: Off-policy maximum entropy deep reinforcement learning with a stochastic actor.
\newblock In \emph{Proceedings of International Conference on Machine Learning}, pages 1861--1870, 2018.

\bibitem[Katz et~al.(2019)Katz, Carlo, and Kim]{katz2019minicheetah}
B.~Katz, J.~D. Carlo, and S.~Kim.
\newblock Mini cheetah: A platform for pushing the limits of dynamic quadruped control.
\newblock In \emph{Proceedings of International Conference on Robotics and Automation}, pages 6295--6301, 2019.

\bibitem[Kim and Oh(2022{\natexlab{a}})]{kim2022offtrc}
D.~Kim and S.~Oh.
\newblock Efficient off-policy safe reinforcement learning using trust region conditional value at risk.
\newblock \emph{IEEE Robotics and Automation Letters}, 7\penalty0 (3):\penalty0 7644--7651, 2022{\natexlab{a}}.

\bibitem[Kim and Oh(2022{\natexlab{b}})]{kim2022trc}
D.~Kim and S.~Oh.
\newblock {TRC}: Trust region conditional value at risk for safe reinforcement learning.
\newblock \emph{IEEE Robotics and Automation Letters}, 7\penalty0 (2):\penalty0 2621--2628, 2022{\natexlab{b}}.

\bibitem[Kuznetsov et~al.(2020)Kuznetsov, Shvechikov, Grishin, and Vetrov]{kuznetsov2020tqc}
A.~Kuznetsov, P.~Shvechikov, A.~Grishin, and D.~Vetrov.
\newblock Controlling overestimation bias with truncated mixture of continuous distributional quantile critics.
\newblock In \emph{Proceedings International Conference on Machine Learning}, pages 5556--5566, 2020.

\bibitem[Lee et~al.(2020)Lee, Hwangbo, Wellhausen, Koltun, and Hutter]{lee2020challenging}
J.~Lee, J.~Hwangbo, L.~Wellhausen, V.~Koltun, and M.~Hutter.
\newblock Learning quadrupedal locomotion over challenging terrain.
\newblock \emph{Science Robotics}, 5\penalty0 (47):\penalty0 eabc5986, 2020.

\bibitem[Lee et~al.(2022)Lee, Paduraru, Mankowitz, Heess, Precup, Kim, and Guez]{lee2022coptidice}
J.~Lee, C.~Paduraru, D.~J. Mankowitz, N.~Heess, D.~Precup, K.-E. Kim, and A.~Guez.
\newblock {CO}pti{DICE}: Offline constrained reinforcement learning via stationary distribution correction estimation.
\newblock In \emph{Proceedings of International Conference on Learning Representations}, 2022.

\bibitem[Liu et~al.(2021)Liu, Liu, Jin, Stone, and Liu]{liu2021cagrad}
B.~Liu, X.~Liu, X.~Jin, P.~Stone, and Q.~Liu.
\newblock Conflict-averse gradient descent for multi-task learning.
\newblock In \emph{Advances in Neural Information Processing Systems}, pages 18878--18890, 2021.

\bibitem[Liu et~al.(2020)Liu, Ding, and Liu]{liu2020ipo}
Y.~Liu, J.~Ding, and X.~Liu.
\newblock {IPO}: Interior-point policy optimization under constraints.
\newblock \emph{Proceedings of the AAAI Conference on Artificial Intelligence}, 34\penalty0 (04):\penalty0 4940--4947, 2020.

\bibitem[Liu et~al.(2022)Liu, Cen, Isenbaev, Liu, Wu, Li, and Zhao]{liu2022cvpo}
Z.~Liu, Z.~Cen, V.~Isenbaev, W.~Liu, S.~Wu, B.~Li, and D.~Zhao.
\newblock Constrained variational policy optimization for safe reinforcement learning.
\newblock In \emph{Proceedings of International Conference on Machine Learning}, pages 13644--13668, 2022.

\bibitem[Meng et~al.(2022)Meng, Zheng, Shi, and Pan]{meng2022offtrpo}
W.~Meng, Q.~Zheng, Y.~Shi, and G.~Pan.
\newblock An off-policy trust region policy optimization method with monotonic improvement guarantee for deep reinforcement learning.
\newblock \emph{IEEE Transactions on Neural Networks and Learning Systems}, 33\penalty0 (5):\penalty0 2223--2235, 2022.

\bibitem[Merel et~al.(2020)Merel, Tunyasuvunakool, Ahuja, Tassa, Hasenclever, Pham, Erez, Wayne, and Heess]{merel2020humanoid}
J.~Merel, S.~Tunyasuvunakool, A.~Ahuja, Y.~Tassa, L.~Hasenclever, V.~Pham, T.~Erez, G.~Wayne, and N.~Heess.
\newblock Catch \& carry: Reusable neural controllers for vision-guided whole-body tasks.
\newblock \emph{ACM Transactions on Graphics}, 39\penalty0 (4), 2020.

\bibitem[Miki et~al.(2022)Miki, Lee, Hwangbo, Wellhausen, Koltun, and Hutter]{miki2022wild}
T.~Miki, J.~Lee, J.~Hwangbo, L.~Wellhausen, V.~Koltun, and M.~Hutter.
\newblock Learning robust perceptive locomotion for quadrupedal robots in the wild.
\newblock \emph{Science Robotics}, 7\penalty0 (62):\penalty0 eabk2822, 2022.

\bibitem[Nam et~al.(2021)Nam, Kim, and Park]{nam2021gmac}
D.~W. Nam, Y.~Kim, and C.~Y. Park.
\newblock {GMAC}: A distributional perspective on actor-critic framework.
\newblock In \emph{Proceedings of International Conference on Machine Learning}, pages 7927--7936, 2021.

\bibitem[Navon et~al.(2022)Navon, Shamsian, Achituve, Maron, Kawaguchi, Chechik, and Fetaya]{navon2022nash}
A.~Navon, A.~Shamsian, I.~Achituve, H.~Maron, K.~Kawaguchi, G.~Chechik, and E.~Fetaya.
\newblock Multi-task learning as a bargaining game.
\newblock \emph{arXiv preprint arXiv:2202.01017}, 2022.

\bibitem[Peng et~al.(2021)Peng, Ma, Abbeel, Levine, and Kanazawa]{peng2021amp}
X.~B. Peng, Z.~Ma, P.~Abbeel, S.~Levine, and A.~Kanazawa.
\newblock {AMP}: Adversarial motion priors for stylized physics-based character control.
\newblock \emph{ACM Transactions on Graphics}, 40\penalty0 (4), 2021.

\bibitem[Precup et~al.(2000)Precup, Sutton, and Singh]{precup2000tdtrace}
D.~Precup, R.~S. Sutton, and S.~P. Singh.
\newblock Eligibility traces for off-policy policy evaluation.
\newblock In \emph{Proceedings of International Conference on Machine Learning}, pages 759--766, 2000.

\bibitem[Qin et~al.(2021)Qin, Chen, and Fan]{qin2021density}
Z.~Qin, Y.~Chen, and C.~Fan.
\newblock Density constrained reinforcement learning.
\newblock In \emph{Proceedings of International Conference on Machine Learning}, pages 8682--8692, 2021.

\bibitem[Ray et~al.(2019)Ray, Achiam, and Amodei]{ray2019safetygym}
A.~Ray, J.~Achiam, and D.~Amodei.
\newblock {Benchmarking Safe Exploration in Deep Reinforcement Learning}.
\newblock 2019.

\bibitem[Rowland et~al.(2018)Rowland, Bellemare, Dabney, Munos, and Teh]{rowland2018dist}
M.~Rowland, M.~Bellemare, W.~Dabney, R.~Munos, and Y.~W. Teh.
\newblock An analysis of categorical distributional reinforcement learning.
\newblock In \emph{Proceedings of International Conference on Artificial Intelligence and Statistics}, pages 29--37, 2018.

\bibitem[Rudin et~al.(2022)Rudin, Hoeller, Reist, and Hutter]{rudin2022walk}
N.~Rudin, D.~Hoeller, P.~Reist, and M.~Hutter.
\newblock Learning to walk in minutes using massively parallel deep reinforcement learning.
\newblock In \emph{Proceedings of Conference on Robot Learning}, pages 91--100, 2022.

\bibitem[Schulman et~al.(2015)Schulman, Levine, Abbeel, Jordan, and Moritz]{schulman2015trpo}
J.~Schulman, S.~Levine, P.~Abbeel, M.~Jordan, and P.~Moritz.
\newblock Trust region policy optimization.
\newblock In \emph{Proceedings of International Conference on Machine Learning}, pages 1889--1897, 2015.

\bibitem[Stooke et~al.(2020)Stooke, Achiam, and Abbeel]{stooke2020pid}
A.~Stooke, J.~Achiam, and P.~Abbeel.
\newblock Responsive safety in reinforcement learning by {PID} lagrangian methods.
\newblock In \emph{Proceedings of International Conference on Machine Learning}, pages 9133--9143, 2020.

\bibitem[Sutton(1988)]{sutton1988learning}
R.~S. Sutton.
\newblock Learning to predict by the methods of temporal differences.
\newblock \emph{Machine learning}, 3\penalty0 (1):\penalty0 9--44, 1988.

\bibitem[Tang et~al.(2022)Tang, Munos, Rowland, Avila~Pires, Dabney, and Bellemare]{tang2022disttd}
Y.~Tang, R.~Munos, M.~Rowland, B.~Avila~Pires, W.~Dabney, and M.~Bellemare.
\newblock The nature of temporal difference errors in multi-step distributional reinforcement learning.
\newblock In \emph{Advances in Neural Information Processing Systems}, pages 30265--30276, 2022.

\bibitem[Thananjeyan et~al.(2021)Thananjeyan, Balakrishna, Nair, Luo, Srinivasan, Hwang, Gonzalez, Ibarz, Finn, and Goldberg]{thananjeyan2021rrl}
B.~Thananjeyan, A.~Balakrishna, S.~Nair, M.~Luo, K.~Srinivasan, M.~Hwang, J.~E. Gonzalez, J.~Ibarz, C.~Finn, and K.~Goldberg.
\newblock Recovery {RL}: Safe reinforcement learning with learned recovery zones.
\newblock \emph{IEEE Robotics and Automation Letters}, 6\penalty0 (3):\penalty0 4915--4922, 2021.

\bibitem[Todorov et~al.(2012)Todorov, Erez, and Tassa]{todorov2012mujoco}
E.~Todorov, T.~Erez, and Y.~Tassa.
\newblock {MuJoCo}: A physics engine for model-based control.
\newblock In \emph{Proceedings of International Conference on Intelligent Robots and Systems}, pages 5026--5033, 2012.

\bibitem[Wang(2018)]{laikago2018}
X.~Wang.
\newblock {Unitree-Laikago Pro}.
\newblock \url{http://www.unitree.cc/e/action/ShowInfo.php?classid=6&id=355}, 2018.

\bibitem[Xie et~al.(2018)Xie, Berseth, Clary, Hurst, and van~de Panne]{xie2018cassie}
Z.~Xie, G.~Berseth, P.~Clary, J.~Hurst, and M.~van~de Panne.
\newblock Feedback control for cassie with deep reinforcement learning.
\newblock In \emph{Proceedings of International Conference on Intelligent Robots and Systems}, pages 1241--1246, 2018.

\bibitem[Xu et~al.(2021)Xu, Liang, and Lan]{xu2021crpo}
T.~Xu, Y.~Liang, and G.~Lan.
\newblock {CRPO}: A new approach for safe reinforcement learning with convergence guarantee.
\newblock In \emph{Proceedings of International Conference on Machine Learning}, pages 11480--11491, 2021.

\bibitem[Yang et~al.(2021)Yang, Simão, Tindemans, and Spaan]{yang2020wcsac}
Q.~Yang, T.~D. Simão, S.~H. Tindemans, and M.~T.~J. Spaan.
\newblock {WCSAC}: Worst-case soft actor critic for safety-constrained reinforcement learning.
\newblock \emph{Proceedings of the AAAI Conference on Artificial Intelligence}, 35\penalty0 (12):\penalty0 10639--10646, 2021.

\bibitem[Yang et~al.(2022)Yang, Simão, Tindemans, and Spaan]{yang2022safety}
Q.~Yang, T.~D. Simão, S.~H. Tindemans, and M.~T.~J. Spaan.
\newblock Safety-constrained reinforcement learning with a distributional safety critic.
\newblock \emph{Machine Learning}, 112:\penalty0 859--887, 2022.

\bibitem[Yang et~al.(2020)Yang, Rosca, Narasimhan, and Ramadge]{yang2020pcpo}
T.-Y. Yang, J.~Rosca, K.~Narasimhan, and P.~J. Ramadge.
\newblock Projection-based constrained policy optimization.
\newblock In \emph{Proceedings of International Conference on Learning Representations}, 2020.

\bibitem[Yu et~al.(2020)Yu, Kumar, Gupta, Levine, Hausman, and Finn]{yu2020pcgrad}
T.~Yu, S.~Kumar, A.~Gupta, S.~Levine, K.~Hausman, and C.~Finn.
\newblock Gradient surgery for multi-task learning.
\newblock In \emph{Advances in Neural Information Processing Systems}, pages 5824--5836, 2020.

\bibitem[Zhang et~al.(2022)Zhang, Lin, Han, Wang, and Lv]{zhang2022cdmpo}
H.~Zhang, Y.~Lin, S.~Han, S.~Wang, and K.~Lv.
\newblock Conservative distributional reinforcement learning with safety constraints.
\newblock \emph{arXiv preprint arXiv:2201.07286}, 2022.

\bibitem[Zhang and Weng(2022)]{zhang2022safedist}
J.~Zhang and P.~Weng.
\newblock Safe distributional reinforcement learning.
\newblock In J.~Chen, J.~Lang, C.~Amato, and D.~Zhao, editors, \emph{Proceedings of International Conference on Distributed Artificial Intelligence}, pages 107--128, 2022.

\end{thebibliography}

\newpage
\appendix
\onecolumn

\section{Algorithm Details}

\subsection{Proof of Theorem \ref{thm:feasible}}
\label{sec:proof_thm2}

We denote the policy parameter space as $\Psi\subseteq\mathbb{R}^d$, the parameter at the $t$th iteration as $\psi_t \in \Psi$, the Hessian matrix as $H(\psi_t)=\nabla_{\psi}^2D_{\mathrm{KL}}(\pi_{\psi_t}||\pi_{\psi})|_{\psi=\psi_t}$, and the $k$th constraint as $F_k(\psi_t)=F_k(\pi_{\psi_t};\alpha)$.
As we focus on the $t$th iteration, the following notations are used for brevity: $H=H(\psi_t)$ and $g_k=\nabla F_k(\psi_t)$.
The proposed gradient integration at $t$th iteration is defined as the following quadratic program (QP):
\begin{equation}
\label{eq:gradient_integ}
g_t = \underset{g}{\mathrm{argmin}}\;\frac{1}{2}g^THg \quad \mathrm{s.t.} \; g_k^Tg + c_k \leq 0 \; \mathrm{for} \; \forall k,
\end{equation}
where $c_k= \mathrm{min}(\sqrt{2\epsilon g_k^T H^{-1}g_k}, F_k(\pi_{\psi}) -d_k + \zeta)$.
In the remainder of this section, we introduce the assumptions and new definitions, discuss the existence of a solution (\ref{eq:gradient_integ}), show the convergence to the feasibility condition for varying step size cases, and provide the proof of Theorem \ref{thm:feasible}.

\textbf{Assumption.}
\textbf{1)} Each $F_k$ is differentiable and convex, \textbf{2)} $\nabla F_k$ is $L$-Lipschitz continuous, \textbf{3)} all eigenvalues of the Hessian matrix $H(\psi)$ are equal or greater than $R \in \mathbb{R}_{>0}$ for $\forall \psi \in \Psi$, and \textbf{4)} $\{\psi|F_k(\psi) + \zeta < d_k \; \mathrm{for} \; \forall k\} \neq \emptyset$.

\textbf{Definition.}
Using the Cholesky decomposition, the Hessian matrix can be expressed as $H = B\cdot B^T$ where $B$ is a lower triangular matrix.
By introducing new terms, $\bar{g}_k:= B^{-1}g_k$ and $b_t := B^T g_t$, the following is satisfied: $g_k^TH^{-1}g_k = ||\bar{g}_k||_2^2$.
Additionally, we define the in-boundary and out-boundary sets as:
\begin{equation*}
\begin{aligned}
\mathrm{IB}_k := \left\{\psi|F_k(\psi) - d_k + \zeta \leq \sqrt{2\epsilon\nabla F_k(\psi)^TH^{-1}(\psi)\nabla F_k(\psi)}\right\}, \\
\mathrm{OB}_k := \left\{\psi|F_k(\psi) - d_k + \zeta \geq \sqrt{2\epsilon\nabla F_k(\psi)^TH^{-1}(\psi)\nabla F_k(\psi)}\right\}.
\end{aligned}
\end{equation*}
The minimum of $||\bar{g}_k||$ in $\mathrm{OB}_k$ is denoted as $m_k$, and the maximum of $||\bar{g}_k||$ in $\mathrm{IB}_k$ is denoted as $M_k$.
Also, $\mathrm{min}_k m_k$ and $\mathrm{max}_k M_k$ are denoted as $m$ and $M$, respectively, and we can say that $m$ is positive.

\begin{lemma}
For all $k$, the minimum value of $m_k$ is positive.
\end{lemma}
\begin{proof}
Assume that there exist $k\in\{1, ..., K\}$ such that $m_{k}$ is equal to zero at a policy parameter $\psi^* \in \mathrm{OB}_k$, i.e., $||\nabla F_k(\psi^*)||=0$.
Since $F_k$ is convex, $\psi^*$ is a minimum point of $F_k$, $\mathrm{min}_{\psi}F_k(\psi) = F_k(\psi^*) < d_k - \zeta$.
However, $F_k(\psi^*) \geq d_k - \zeta$ as $\psi^* \in \mathrm{OB}_k$, so $m_k$ is positive due to the contradiction.
Hence, the minimum of $m_k$ is also positive.
\end{proof}

\begin{lemma}
\label{lemma:QP feasibility}
A solution of (\ref{eq:gradient_integ}) always exists.
\end{lemma}
\begin{proof}
There exists a policy parameter $\hat{\psi} \in \{\psi|F_k(\psi) + \zeta < d_k \; \mathrm{for} \; \forall k\}$ due to the assumptions.
Let $g=\psi-\psi_t$.
Then, the following inequality holds.
\begin{equation*}
\begin{aligned}
g_k^T(\psi - \psi_t) + c_k &\leq g_k^T(\psi - \psi_t) + F_k(\psi_t) + \zeta - d_k \leq F_k(\psi) +\zeta - d_k. && (\because F_k \text{ is convex.}) \\
\Rightarrow g_k^T(\hat{\psi} - \psi_t) + c_k &\leq F_k(\hat{\psi}) + \zeta - d_k < 0 \text{ for } \forall k.
\end{aligned}
\end{equation*}
Since $\hat{\psi} - \psi_t$ satisfies all constraints of (\ref{eq:gradient_integ}), the feasible set is non-empty and convex.
Also, $H$ is positive definite, so the QP has a unique solution.
\end{proof}
Lemma \ref{lemma:QP feasibility} shows the existence of solution of (\ref{eq:gradient_integ}).
Now, we show the convergence of the proposed gradient integration method in the case of varying step sizes.
\begin{lemma}
\label{lemma:varying step size}
If $\sqrt{2\epsilon}M \leq \zeta$ and a policy is updated by $\psi_{t+1}=\psi_t + \beta_t g_t$, where $0 < \beta_t < \frac{2\sqrt{2\epsilon}mR}{L||b_t||^2}$ and $\beta_t\leq1$, the policy satisfies $F_k(\psi) \leq d_k\;$ for $\forall k$ within a finite time.
\end{lemma}
\begin{proof}
We can reformulate the step size as $\beta=\frac{2\sqrt{2\epsilon}mR}{L||b_t||^2}\beta'_t$, where $\beta'_t \leq \frac{L||b_t||^2}{2\sqrt{2\epsilon}mR}$ and $0 < \beta'_t < 1$.
Since the eigenvalues of $H$ is equal to or bigger than $R$ and $H$ is symmetric and positive definite, $\frac{1}{R}I - H^{-1}$ is positive semi-definite.
Hence, $x^T H^{-1}x \leq \frac{1}{R}||x||^2$ is satisfied.
Using this fact, the following inequality holds:
\begin{equation*}
\begin{aligned}
F_k(\psi_t + \beta_t g_t) - F_k(\psi_t) &\leq \beta_t\nabla F_k(\psi_t)^T g_t + \frac{L}{2}||\beta_t g_t||^2 && (\because \nabla F_k \; \text{is $L$-Lipschitz continuous.})\\
&=\beta_t g_k^T g_t + \frac{L}{2}\beta_t^2||g_t||^2 \\
&=\beta_t g_k^T g_t + \frac{L}{2}\beta_t^2 b_t^T H^{-1} b_t && (\because g_t=B^{-T}b_t)\\
&\leq -\beta_t c_k + \frac{L}{2R}\beta_t^2 ||b_t||^2. && (\because g_k^Tg_t + c_k\leq 0)
\end{aligned}
\end{equation*}
Now, we will show that $\psi$ enters $\mathrm{IB}_k$ in a finite time for $\forall\psi\in\mathrm{OB}_k$ and that the $k$th constraint is satisfied for $\forall\psi \in \mathrm{IB}_k$.
Thus, we divide into two cases, \textbf{1)} $\psi_t \in \mathrm{OB}_k$ and \textbf{2)} $\psi_t \in \mathrm{IB}_k$.
For the first case, $c_k = \sqrt{2\epsilon}||\bar{g}_k||$, so the following inequality holds:
\begin{equation}
\label{ieq:decrease}
\begin{aligned}
F_k(\psi_t + \beta_t g_t) - F_k(\psi_t) &\leq \beta_t\left(-\sqrt{2\epsilon} ||\bar{g}_k|| + \frac{L}{2R}\beta_t ||b_t||^2\right) \\
&\leq \beta_t\sqrt{2\epsilon}\left(-||\bar{g}_k|| + m\beta'_t\right) \\
&\leq \beta_t\sqrt{2\epsilon}m(\beta'_t - 1) < 0.
\end{aligned}
\end{equation}
The value of $F_k$ decreases strictly with each update step according to (\ref{ieq:decrease}).
Hence, $\psi_t$ can reach $\mathrm{IB}_k$ by repeatedly updating the policy.
We now check whether the constraint is satisfied for the second case.
For the second case, the following inequality holds by applying $c_k = F_k(\psi_t) - d_k + \zeta$:
\begin{equation*}
\begin{aligned}
&F_k(\psi_t + \beta_t g_t) - F_k(\psi_t) \leq \beta_t d_k - \beta_t F_k(\psi_t) - \beta_t\zeta + \frac{L}{2R}\beta_t^2 ||b_t||^2 \\
\Rightarrow& F_k(\psi_t + \beta_t g_t) - d_k \leq (1 - \beta_t)(F_k(\psi_t) - d_k) + \beta_t(-\zeta + \sqrt{2\epsilon}m\beta'_t).
\end{aligned}
\end{equation*}
Since $\psi_t \in \mathrm{IB}_k$,
\begin{equation*}
F_k(\psi_t) - d_k \leq \sqrt{2\epsilon}||\bar{g}_k|| - \zeta \leq \sqrt{2\epsilon}M - \zeta \leq 0.
\end{equation*}
Since $m \leq M$ and $\beta'_t < 1$,
\begin{equation*}
-\zeta + \sqrt{2\epsilon}m\beta'_t < -\zeta + \sqrt{2\epsilon}M \leq 0.
\end{equation*}
Hence, $F_k(\psi_t + \beta_t g_t) \leq d_k$, which means that the $k$th constraint is satisfied if $\psi_t \in \mathrm{IB}_k$.
As $\psi_t$ reaches $\mathrm{IB}_k$ for $\forall k$ within a finite time according to (\ref{ieq:decrease}), the policy can satisfy all constraints within a finite time.
\end{proof}
Lemma \ref{lemma:varying step size} shows the convergence to the feasibility condition in the case of varying step sizes.
We introduce a lemma, which shows $||b_t||$ is bounded by $\sqrt{\epsilon}$, and finally show the proof of Theorem \ref{thm:feasible}, which can be considered a special case of varying step sizes.

\begin{lemma}
\label{lemma:b_k}
There exists $T \in \mathbb{R}_{>0}$ such that $||b_t|| \leq T\sqrt{\epsilon}$.
\end{lemma}

\begin{proof}
Let us define the following sets:
\begin{equation}
\begin{aligned}
I &:= \{k|F_k(\psi_t) + \zeta - d_k < 0\}, \; O := \{k|F_k(\psi_t) + \zeta - d_k > 0\}, \\
U &:= \{k|F_k(\psi_t) + \zeta - d_k = 0\}, \; C(\epsilon) := \{g|g_k^Tg + c_k(\epsilon) \leq 0 \; \forall k\}, \\
I_G &:= \{g | g_k^Tg + F_k(\psi_t) + \zeta - d_k \leq 0 \; \forall k\in I\}, 
\end{aligned}
\end{equation}
where $c_k(\epsilon) = \mathrm{min}(\sqrt{2\epsilon g_k^T H^{-1}g_k}, F_k(\pi_{\psi}) -d_k + \zeta)$.
Using these sets, the following vectors can be defined: $g(\epsilon) := \mathrm{argmin}_{g \in C(\epsilon)} \frac{1}{2}g^THg$, $b(\epsilon) := B^Tg(\epsilon).$
Now, we will show that $||b(\epsilon)||$ is bounded above and $||b(\epsilon)|| \propto \sqrt{\epsilon}$ for sufficiently small $\epsilon > 0$.

First, the following is satisfied for a sufficiently large $\epsilon$: 
\begin{equation}
C(\epsilon) = \{g| g_k^Tg +F_k(\psi_t) + \zeta - d_k \leq 0 \; \forall k\}.
\end{equation}
Since $\hat{\psi} - \psi_t \in C(\epsilon)$, where $\hat{\psi}$ is defined in Lemma \ref{lemma:QP feasibility}, $||b(\epsilon)|| \leq \frac{1}{2}(\hat{\psi}-\psi_t)^T H(\hat{\psi}-\psi_t)$ for $\forall \epsilon$.
Therefore, $||b(\epsilon)||$ is bounded above.

Second, let us define the following trust region size:
\begin{equation}
\hat{\epsilon} := \frac{1}{2}\left(\mathrm{min}_{k \in O}\frac{F_k(\psi_t) + \zeta - d_k}{||\bar{g_k}||}\right)^2 > 0.
\end{equation}
if , $\epsilon \leq \hat{\epsilon}$, the following is satisfied:
\begin{equation}
C(\epsilon) =  I_G \cap \{g|g_k^Tg + \sqrt{2\epsilon}||\bar{g_k}|| \leq 0 \; \forall k \in O, \; g_k^Tg \leq 0 \; \forall k\in U\} \neq \phi.
\end{equation}
Thus, $O_G(\epsilon) := \{g|g_k^Tg + \sqrt{2\epsilon}||\bar{g_k}|| \leq 0 \; \forall k \in O, \; g_k^Tg \leq 0 \; \forall k\in U\}$ is not empty.
If we define $\hat{g}(\epsilon) := \mathrm{argmin}_{g \in O_G(\epsilon)} \frac{1}{2}g^THg$, the following is satisfied: \begin{equation}
    \hat{g}(\epsilon) = \sqrt{\frac{\epsilon}{\hat{\epsilon}}}\hat{g}(\hat{\epsilon}) \; \text{for} \; 0 \leq \epsilon \leq \hat{\epsilon}.
\end{equation}
Then, if $\epsilon \leq \hat{\epsilon}\left(\mathrm{min}_{k \in I} (d_k - \zeta - F_k(\psi_t))/g_k^T \hat{g}(\hat{\epsilon}))\right)^2$, $\hat{g}(\epsilon) = g(\epsilon)$ since $\hat{g}(\epsilon) \in I_G$.
Consequently, by defining a trust region size:
\begin{equation}
\epsilon^* := \hat{\epsilon} \cdot \mathrm{min}\left(1,  \left(\mathrm{min}_{k \in I} \frac{d_k - \zeta - F_k(\psi_t)}{g_k^T \hat{g}(\hat{\epsilon})}\right)^2\right) > 0,
\end{equation}
$g(\epsilon) = \sqrt{\epsilon/\hat{\epsilon}}\hat{g}(\hat{\epsilon})$ for $\epsilon \leq \epsilon^*$.
Therefore, $||b(\epsilon)|| \propto \sqrt{\epsilon}$ if $\epsilon \leq \epsilon^*$.

Finally, since $||b(\epsilon)||$ is bounded above and proportional to $\sqrt{\epsilon}$ for sufficiently small $\epsilon$, there exist a constant $T$ such that $||b(\epsilon)|| \leq T\sqrt{\epsilon}$.
\end{proof}

\feasible*
\begin{proof}
The proposed step size is $\beta_t = \mathrm{min}(1, \sqrt{2\epsilon}/||b_t||)$, and the sufficient conditions that guarantee the convergence according to Lemma \ref{lemma:varying step size} are followings:
\begin{equation*}
\sqrt{2\epsilon}M \leq \zeta \;\mathrm{and}\; 0 < \beta_t \leq 1 \;\mathrm{and}\; \beta_t < \frac{2\sqrt{2\epsilon}mR}{L||b_t||^2}.
\end{equation*}
The second condition is self-evident.
To satisfy the third condition, the proposed step size $\beta_t$ should satisfy the followings:
\begin{equation*}
\begin{aligned}
&\frac{\sqrt{2\epsilon}}{||b_t||} < \frac{2\sqrt{2\epsilon}mR}{L||b_t||^2} \; \Leftrightarrow \; ||b_t|| < \frac{2mR}{L}.
\end{aligned}
\end{equation*}
If $\epsilon < 4((mR)/(LT))^2$, the following inequality holds:
\begin{equation*}
\begin{aligned}
\sqrt{\epsilon} < \frac{2mR}{LT} \; \Rightarrow \; ||b_t|| \leq T\sqrt{\epsilon} < \frac{2mR}{L}. &&(\because \; \text{Lemma \ref{lemma:b_k}}.)
\end{aligned}
\end{equation*}
Hence, if $\epsilon \leq E=\frac{1}{2} \mathrm{min}(\frac{\zeta^2}{2M^2}, 4(\frac{mR}{LT})^2)$, the sufficient conditions are satisfied.
\end{proof}

\subsection{Toy Example for Gradient Integration Method}
\label{sec:toy_example_fig_grad_integ}

The problem of the toy example in Figure \ref{fig:grad integration} is defined as:
\begin{equation}
\begin{aligned}
&\underset{x_1,x_2}{\mathrm{minimize}}\sqrt{(\sqrt{3}x_1 + x_2 + 2)^2 + 4(x_1 - \sqrt{3}x_2 + 4)^2} \quad \mathbf{s.t.} \; x_1 \geq 0, \; x_1 - 2x_2 \leq 0,
\end{aligned}
\end{equation}
where there are two linear constraints.
The initial points for the naive and gradient integration methods are $x_1=-2.5$ and $x_2=-3.0$, which do not satisfied the two constraints.
We use the Hessian matrix for the trust region as identity matrix and the trust region size as $0.5$ in both methods.
The naive method minimizes the constraints in order from the first to the second constraint.

\subsection{Analysis of Worst-Case Time to Satisfy All Constraints}
\label{sec:worst case time}

To analyze the sample complexity, we consider a tabular MDP and use softmax policy parameterization as follows (for more details, see \citep{xu2021crpo}):
\begin{equation}
\pi_\psi(a|s) := \frac{\exp{\psi(s, a)}}{\sum_{a'}\exp{\psi(s, a')}} \; \forall (s, a) \in S \times A.
\end{equation}
According to \citet{agarwal2021theory}, the natural policy gradient (NPG) update is as follows:
\begin{equation}
\psi_{t+1} = \psi_t + \beta A^{\pi_{\psi_t}}, \; \pi_{\psi_{t+1}}(a|s) = \pi_{\psi_t}(a|s) \frac{\exp(\beta A^{\pi_{\psi_t}}(s,a))}{Z_t(s)},
\end{equation}
where $\beta$ is a step size, $A^{\pi_{\psi_t}} \in \mathbb{R}^{|S||A|}$ is the vector expression of the advantage function, and $Z_t(s) = \sum_a \pi_{\psi_t}(a|s) \exp{(\beta A^{\pi_{\psi_t}}(s,a))/(1 - \gamma)}$.
Analyzing the sample complexity of trust region-based methods is challenging since their stepsize is not fixed, so we modify the gradient integration method to use the NPG as follows:
\begin{equation}
\begin{aligned}
&g^* = \mathrm{argmin}_g \frac{1}{2}g^T Hg \; \mathrm{s.t.} \; g_k^Tg + c_k \leq 0 \; \forall k \in \{k|F_k(\pi_{\psi_t};\alpha) > d_k\}, \\
&\psi_{t+1} = \psi_{t} + \beta g^*/||g^*||_2.
\end{aligned}
\end{equation}
In the remainder, we abbreviate $\pi_{\psi_t}$, $A^{\pi_{\psi_t}}$, and $F_k(\pi_{\psi_t};\alpha)$ as $\pi_t$, $A^t$, and $F_k(\pi_t)$, respectively.
Since $g^*$ always exists due to Lemma \ref{lemma:QP feasibility}, we can write the policy using Lagrange multipliers $\lambda_k^t \geq 0$ as follows:
\begin{equation}
\begin{aligned}
&\psi_{t+1} = \psi_t - \beta\sum_{k}\lambda^t_k A_{C_k}^t/W_t, \; W_t:=||\sum_{k} \lambda^t_k A_{C_k}^t||_2, \\
&\pi_{t+1}(a|s) = \pi_{t}(a|s) \exp{\left(-\frac{\beta}{W_t}\sum_{k}\lambda^t_k A_{C_k}^t(s, a)\right)}/Z_t(s),
\end{aligned}
\end{equation}
where $Z_t(s)$ is a normalization factor, and $\lambda_k^t=0$ for $F_k(\pi_t) \leq d_k$.
The naive approach can also be written as above, except that $\lambda^t$ is a one-hot vector, where $i$-th value $\lambda_i^t$ is one only for corresponding to the randomly selected constraint.
Then, we can get the followings:
\begin{equation}
\begin{aligned}
\sum_k \lambda^t_k & (F_k(\pi_{t+1}) - F_k(\pi_{t}))/W_t = \frac{1}{1 - \gamma} \mathbb{E}_{s \sim d^{\pi_{t+1}}} \left[\sum_a \pi_{t+1}(a|s)\sum_k\lambda^t_k A_{C_k}^t(s, a)/W_t \right] \\
&= -\frac{1}{\beta(1 - \gamma)} \mathbb{E}_{s \sim d^{\pi_{t+1}}} \left[\sum_a \pi_{t+1}(a|s) \log \frac{\pi_{t+1}(a|s)Z_t(s)}{\pi_t(a|s)} \right] \\
&= -\frac{1}{\beta(1 - \gamma)} \mathbb{E}_{s \sim d^{\pi_{t+1}}} \left[D_\mathrm{KL}(\pi_{t+1}(\cdot|s) || \pi_{t}(\cdot|s)) + \log Z_t(s) \right] \\
&\leq - \frac{1}{\beta(1 - \gamma)} \mathbb{E}_{s \sim d^{\pi_{t+1}}} \left[\log{Z_t(s)}\right] \;\;\;\quad\qquad\qquad\qquad\qquad (\because D_\mathrm{KL}(\pi'||\pi) \geq 0)\\
&\leq - \frac{1}{\beta} \mathbb{E}_{s \sim \rho} \left[\log{Z_t(s)}\right], \qquad\qquad\qquad\qquad\qquad\qquad\qquad (\because ||d^{\pi}/\rho||_\infty \geq 1-\gamma)
\end{aligned}
\end{equation}
We can also get the followings by using the Lemma 7 in \citet{xu2021crpo}:
\begin{equation}
\begin{aligned}
\sum_k \lambda^t_k(F_k(\pi_*) - F_k(\pi_{t}))/W_t \geq& -\frac{1}{\beta(1-\gamma)}\mathbb{E}_{s \sim d^*} \left[D_\mathrm{KL}(\pi_*||\pi_{t}) - D_\mathrm{KL}(\pi_*||\pi_{t+1})\right] \\
&- \sum_k \lambda^t_k \frac{2 \beta C_\mathrm{max}}{(1 - \gamma)^2 W_t},
\end{aligned}
\end{equation}
where $\pi_*$ is an optimal policy, and $C_\mathrm{max}$ is the maximum value of costs.
If $\lambda^t_k > 0$, $F_k(\pi_{t}) - F_k(\pi_*) > \zeta$.
Thus, $\sum_k \lambda^t_k(F_k(\pi_*) - F_k(\pi_{t})) \leq -\zeta\sum_k \lambda^t_k$.
If the policy does not satisfy the constraints until $T$ step, the following inequality holds by summing the above inequalities from $t=0$ to $T$:
\begin{equation}
\beta (1-\gamma) (\zeta - \frac{2 \beta C_\mathrm{max}}{(1 - \gamma)^2}) \sum_{t=0}^T \sum_i \lambda^t_i / W_t \leq \mathbb{E}_{s \sim d^*} \left[D_\mathrm{KL}(\pi_*||\pi_{0})\right].
\end{equation}
Let denote $\frac{1}{T}\sum_{t=0}^T \sum_i \lambda^t_i / W_t$ as $\mathbb{E}_t[\sum_i \lambda^t_i / W_t]$, and we can get $W_t = ||\sum_k \lambda^t_k A_{C_k}^t||_2 \leq \sum_k \lambda^t_k 2 |S||A|C_\mathrm{max}/(1-\gamma)$.
Then, the maximum $T$ can be expressed as:
\begin{equation}
T \leq  \frac{D_\mathrm{KL}}{\beta (1-\gamma) (\zeta - \frac{2 \beta C_\mathrm{max}}{(1 - \gamma)^2}) \mathbb{E}_t[\sum_i \lambda^t_i / W_t]} \leq \frac{2|S||A|C_\mathrm{max}D_\mathrm{KL}}{\beta (1-\gamma)^2\zeta - 2\beta^2C_\mathrm{max}} =: T_\mathrm{max},
\end{equation}
where we abbreviate $\mathbb{E}_{s \sim d^*} \left[D_\mathrm{KL}(\pi_*||\pi_{0})\right]$ as $D_\mathrm{KL}$.
Finally, the policy can reach the feasible region within $T_\mathrm{max}$ steps.

The worst-case time of the naive approach is the same as the above equation, except for the $\lambda$ part.
In the naive approach, $\lambda^t$ is a one-hot vector, as mentioned earlier.
In other words, only $\mathbb{E}_t[\sum_i\lambda_i^t/W_t] = \mathbb{E}_t[\sum_i\lambda_i^t/||\sum_i\lambda_i^tA_{C_k}^t||_2]$ is different.
Let us assume that the advantage vector follows a normal distribution. 
Then, the variance of $\sum_i\lambda_i^tA_{C_k}^t$ is smaller for $\lambda^t$ with distributed values than for one-hot values.
Then, the reciprocal of the 2-norm becomes larger, resulting in a decrease in the worst-case time.
From this perspective, the gradient integration method has a benefit over the naive approach as it reduces the variance of the advantage vector.
Even though we cannot officially say that the worst-case time of the proposed method is smaller than the naive method because the advantage vector does not follow the normal distribution, we can deliver our insight on the benefit of gradient integration method.

\subsection{Proof of Theorem \ref{thm:target}}
\label{sec:proof_contraction}

In this section, we show that a sequence, $Z_{k+1} = \mathcal{T}_\lambda^{\mu, \pi}Z_k$, converges to the $Z^{\pi}_R$.
First, we rewrite the operator $\mathcal{T}_\lambda^{\mu, \pi}$ for random variables to an operator for distributions and show that the operator is contractive.
Finally, we show that $Z_R^{\pi}$ is the unique fixed point.

Before starting the proof, we introduce useful notions and distance metrics.
As the return $Z_R^{\pi}(s, a)$ is a random variable, we define the distribution of $Z_R^{\pi}(s, a)$ as $\nu_R^\pi(s,a)$.
Let $\eta$ be the distribution of a random variable $X$.
Then, we can express the distribution of affine transformation of random variable, $aX + b$, using the \emph{pushforward} operator, which is defined by \citet{rowland2018dist}, as $(f_{a,b})_{\#}(\eta)$.
To measure a distance between two distributions, \citet{bdr2022} has defined the distance $l_p$ as follows:
\begin{equation}
l_p(\eta_1, \eta_2) := \left(\int_{\mathbb{R}}\left|F_{\eta_1}(x) - F_{\eta_2}(x)\right|^p dx\right)^{1/p}, 
\end{equation}
where $F_{\eta}(x)$ is the cumulative distribution function.
This distance is $1/p$-homogeneous, regular, and $p$-convex (see Section 4 of \citet{bdr2022} for more details).
For functions that map state-action pairs to distributions, a distance can be defined as \citep{bdr2022}: $\bar{l}_p(\nu_1, \nu_2):=\mathrm{sup}_{(s,a)\in S\times A}l_p(\nu_1(s,a), \nu_2(s,a))$.
Then, we can rewrite the operator $\mathcal{T}_\lambda^{\mu,\pi}$ for random variables in (\ref{eq:distributional_operator}) as an operator for distributions as below.
\begin{equation}
\label{eq:distributional_operator2}
\begin{aligned}
&\mathcal{T}^{\mu, \pi}_\lambda\nu(s,a) := \frac{1-\lambda}{\mathcal{N}}\sum_{i=0}^{\infty}\lambda^i\\
& \times\mathbb{E}_{\mu}\left[\left(\prod_{j=1}^i\eta(s_j, a_j)\right)\mathbb{E}_{a'\sim\pi(\cdot|s_{i+1})}\left[(f_{\gamma^{i+1}, \sum_{t=0}^i\gamma^t r_t})_{\#} (\nu(s_{i+1}, a'))\right]\Big|s_0=s, a_0=a\right],
\end{aligned}
\end{equation}
where $\eta(s, a)=\frac{\pi(a|s)}{\mu(a|s)}$ and $\mathcal{N}$ is a normalization factor.
Since the random variable $Z(s,a)$ and the distribution $\nu(s,a)$ is equivalent, the operators in (\ref{eq:distributional_operator}) and (\ref{eq:distributional_operator2}) are also equivalent.
Hence, we are going to show the proof of Theorem \ref{thm:target} using (\ref{eq:distributional_operator2}) instead of (\ref{eq:distributional_operator}).
We first show that the operator $\mathcal{T}^{\mu, \pi}_\lambda$ has a contraction property.
\begin{lemma}
\label{lemma: contractive}
Under the distance $\bar{l}_p$ and the assumption that the state, action, and reward spaces are finite, $\mathcal{T}^{\mu, \pi}_\lambda$ is $\gamma^{1/p}$-contractive.
\end{lemma}
\begin{proof}
First, the operator can be rewritten using summation as follows.
\begin{equation}
\begin{aligned}
\mathcal{T}^{\mu, \pi}_\lambda\nu(s,a) &= \frac{1-\lambda}{\mathcal{N}}\sum_{i=0}^{\infty}\lambda^i\sum_{a'\in A}\sum_{(s_0, a_0, r_0, ..., s_{i+1})}\mathrm{Pr}_\mu(\underbrace{s_0, a_0, r_0, ..., s_{i+1}}_{=:\tau})\left(\prod_{j=1}^i\eta(s_j, a_j)\right)\\
&\quad\quad\quad \times \pi(a'|s_{i+1})(f_{\gamma^{i+1}, \sum_{t=0}^i\gamma^t r_t})_{\#} (\nu(s_{i+1}, a'))\\
&=\frac{1-\lambda}{\mathcal{N}}\sum_{i=0}^{\infty}\lambda^i\sum_{a'\in A}\sum_\tau\mathrm{Pr}_\mu(\tau)\left(\prod_{j=1}^i\eta(s_j, a_j)\right)\pi(a'|s_{i+1})\sum_{s'\in S}\mathbf{1}_{s'=s_{i+1}} \\
&\quad\quad\quad \times \sum_{r'_{0:i}}\left(\prod_{k=0}^{i}\mathbf{1}_{r'_k=r_k}\right)(f_{\gamma^{i+1}, \sum_{t=0}^i\gamma^t r'_t})_{\#} (\nu(s', a'))\\
&= \frac{1-\lambda}{\mathcal{N}}\sum_{i=0}^{\infty}\lambda^i\sum_{a'\in A}\sum_{s'\in S}\sum_{r'_{0:i}}(f_{\gamma^{i+1}, \sum_{t=0}^i\gamma^t r'_t})_{\#} (\nu(s', a'))\\
& \quad\quad\quad \times\underbrace{\mathbb{E}_{\mu}\left[\left(\prod_{j=1}^i\eta(s_j, a_j)\right)\pi(a'|s_{i+1})\mathbf{1}_{s'=s_{i+1}}\left(\prod_{k=0}^{i}\mathbf{1}_{r'_k=r_k}\right)\right]}_{=: w_{s', a', r'_{0:i}}}\\
&= \frac{1-\lambda}{\mathcal{N}}\sum_{i=0}^{\infty}\sum_{s'\in S}\sum_{a'\in A}\sum_{r'_{0:i}}\lambda^iw_{s', a', r'_{0:i}}(f_{\gamma^{i+1}, \sum_{t=0}^i\gamma^t r'_t})_{\#} (\nu(s', a')).\\
\end{aligned}
\end{equation}
Since the sum of weights of distributions should be one, we can find the normalization factor $\mathcal{N}=(1-\lambda)\sum_{i=0}^{\infty}\sum_{s\in S}\sum_{a\in A}\sum_{r_{0:i}}\lambda^iw_{s, a, r_{0:i}}$.
Then, the following inequality can be derived using the homogeneity, regularity, and convexity of $l_p$:
\begin{equation}
\begin{aligned}
l^p_p&(\mathcal{T}^{\mu, \pi}_\lambda\nu_1(s,a), \mathcal{T}^{\mu, \pi}_\lambda\nu_2(s,a)) \\
&= l^p_p\left(\frac{1-\lambda}{\mathcal{N}}\sum_{i=0}^{\infty}\sum_{s\in S}\sum_{a\in A}\sum_{r_{0:i}}\lambda^iw_{s, a, r_{0:i}}(f_{\gamma^{i+1}, \sum_{t=0}^i\gamma^tr_t})_{\#} (\nu_1(s, a)), \right. \\
&\qquad\qquad\qquad \left.\frac{1-\lambda}{\mathcal{N}}\sum_{i=0}^{\infty}\sum_{s\in S}\sum_{a\in A}\sum_{r_{0:i}}\lambda^iw_{s, a, r_{0:i}}(f_{\gamma^{i+1}, \sum_{t=0}^i\gamma^tr_t})_{\#} (\nu_2(s, a))\right) \\
&\leq \sum_{i=0}^{\infty}\sum_{s\in S}\sum_{a\in A}\sum_{r_{0:i}}\frac{(1-\lambda)\lambda^iw_{s, a, r_{0:i}}}{\mathcal{N}}l^p_p\left((f_{\gamma^{i+1}, \sum_{t=0}^i\gamma^tr_t})_{\#} (\nu_1(s, a)), \right.\\
&\left.\qquad\qquad\qquad(f_{\gamma^{i+1}, \sum_{t=0}^i\gamma^tr_t})_{\#} (\nu_2(s, a))\right) \\
&\leq \sum_{i=0}^{\infty}\sum_{s\in S}\sum_{a\in A}\sum_{r_{0:i}}\frac{(1-\lambda)\lambda^iw_{s, a, r_{0:i}}}{\mathcal{N}}l^p_p\left((f_{\gamma^{i+1}, 0})_{\#} (\nu_1(s, a)), (f_{\gamma^{i+1}, 0})_{\#} (\nu_2(s, a))\right) \\
&= \sum_{i=0}^{\infty}\sum_{s\in S}\sum_{a\in A}\sum_{r_{0:i}}\frac{(1-\lambda)\lambda^iw_{s, a, r_{0:i}}}{\mathcal{N}}\gamma^{i+1}l^p_p\left( \nu_1(s, a), \nu_2(s, a)\right) \\
&\leq \sum_{i=0}^{\infty}\sum_{s\in S}\sum_{a\in A}\sum_{r_{0:i}}\frac{(1-\lambda)\lambda^iw_{s, a, r_{0:i}}}{\mathcal{N}}\gamma^{i+1}\left(\bar{l}_p\left(\nu_1, \nu_2 \right)\right)^p \\
&\leq \gamma\left(\bar{l}_p\left(\nu_1, \nu_2 \right)\right)^p. \\
\end{aligned}
\end{equation}
Therefore, $\bar{l}_p\left(\mathcal{T}^{\mu, \pi}_\lambda\nu_1, \mathcal{T}^{\mu, \pi}_\lambda\nu_2 \right) \leq \gamma^{1/p} \bar{l}_p\left(\nu_1, \nu_2 \right)$.
\end{proof}

By the Banach's fixed point theorem, the operator $\mathcal{T}^{\mu, \pi}_\lambda$ has a unique fixed distribution.
We now show that the fixed distribution is $\nu^\pi_R$.
\begin{lemma}
\label{lemma: fixed point}
The fixed distribution of the operator $\mathcal{T}^{\mu, \pi}_\lambda$ is $\nu^\pi_R$.
\end{lemma}
\begin{proof}
From the definition of $Z_R^{\pi}$, the following equality holds \citep{rowland2018dist}: $\nu_R^{\pi}(s,a) = \mathbb{E}_{\pi}\left[(f_{\gamma, r})_{\#}(\nu_R^{\pi}(s', a'))\right]$.
Then, it can be shown that $\nu_R^\pi$ is the fixed distribution by applying the operator $\mathcal{T}^{\mu, \pi}_\lambda$ to $\nu_R^\pi$:
\begin{equation}
\begin{aligned}
&\mathcal{T}^{\mu, \pi}_\lambda\nu_R^\pi(s,a) = \frac{1-\lambda}{\mathcal{N}}\sum_{i=0}^{\infty}\lambda^i\\
& \times\mathbb{E}_{\mu}\left[\left(\prod_{j=1}^i\eta(s_j, a_j)\right)\mathbb{E}_{a'\sim\pi(\cdot|s_{i+1})}\left[(f_{\gamma^{i+1}, \sum_{t=0}^i\gamma^t r_t})_{\#} (\nu_R^{\pi}(s_{i+1}, a'))\right]\Big|s_0=s, a_0=a\right] \\
&=\frac{1-\lambda}{\mathcal{N}}\sum_{i=0}^{\infty}\lambda^i\mathbb{E}_{\pi}\left[(f_{\gamma^{i+1}, \sum_{t=0}^i\gamma^t r_t})_{\#} (\nu_R^{\pi}(s_{i+1}, a_{i+1}))\Big|s_0=s, a_0=a\right]\\
& = \frac{1-\lambda}{\mathcal{N}}\sum_{i=0}^{\infty}\lambda^i \nu_R^{\pi}(s, a) = \nu_R^{\pi}(s, a).
\end{aligned}
\end{equation}
\end{proof}

\converge*
\begin{proof}
The operator $\mathcal{T}^{\mu, \pi}_\lambda$ is $\gamma^{1/p}$-contractive under the distance $\bar{l}_p$ according to Lemma \ref{lemma: contractive}.
Also, the fixed distribution of the operator is $\nu_R^\pi$, which is equivalent to $Z_R^\pi$, according to Lemma \ref{lemma: fixed point}.
By the Banach's fixed point theorem, the sequence, $Z_{k+1}(s,a) = \mathcal{T}^{\mu, \pi}_\lambda Z_k(s,a) \; \forall (s,a)$, converges to the fixed distribution of the operator, $Z^\pi_R$.
\end{proof}

\subsection{Pseudocode of TD($\lambda$) Target Distribution}
\label{sec:pseudocode_target}

We provide the pseudocode for calculating TD($\lambda$) target distribution for the reward critic in Algorithm \ref{algo:target distribution}.
The target distribution for the cost critics can also be obtained by simply replacing the reward part with the cost.

\begin{algorithm}[h]
\caption{TD($\lambda$) Target Distribution}
\label{algo:target distribution}
\begin{algorithmic}
\STATE {\bfseries Input:} Policy network $\pi_{\psi}$, critic network $Z_{\theta}^{\pi}$, and trajectory $\{(s_t, a_t, \mu(a_t|s_t), r_t, d_t, s_{t+1})\}_{t=1}^{T}$.
\STATE Sample an action $a'_{T+1}\sim \pi_{\psi}(s_{T+1})$ and get $\hat{Z}_{T}^{\mathrm{tot}} = r_T + (1 - d_T)\gamma Z_{\theta}^{\pi}(s_{T+1}, a'_{T+1})$.
\STATE Initialize the total weight $w_{\mathrm{tot}}=\lambda$.
\FOR{$t=T$ {\bfseries to} $1$}
    \STATE Sample an action $a'_{t+1}\sim \pi_{\psi}(s_{t+1})$ and get $\hat{Z}_{t}^{(1)} = r_t + (1 - d_t)\gamma Z_{\theta}^{\pi}(s_{t+1}, a'_{t+1})$. 
    \STATE Set the current weight $w= 1- \lambda$. 
    \STATE Combine the two targets, $(\hat{Z}_{t}^{(1)}, w)$ and $(\hat{Z}_{t}^{(\mathrm{tot})}, w_\mathrm{tot})$, and sort the combined target according to the positions of atoms.
    \STATE Build the CDF of the combined target by accumulating the weights at each atom.
    \STATE Project the combined target into a quantile distribution with $M'$ atoms, which is $\hat{Z}_{t}^{(\mathrm{proj})}$, using the CDF (find the atom positions corresponding to each quantile). 
    \STATE Update $\hat{Z}_{t-1}^{(\mathrm{tot})} = r_{t-1} + (1-d_{t-1})\gamma\hat{Z}_{t}^{(\mathrm{proj})}$ and $w_{\mathrm{tot}}=\lambda\frac{\pi_{\psi}(a_t|s_t)}{\mu(a_t|s_t)}(1-d_{t-1})(1-\lambda + w_{\mathrm{tot}})$.
\ENDFOR
\STATE \textbf{Return} $\{\hat{Z}_{t}^{(\mathrm{proj})}\}_{t=1}^{T}$.
\end{algorithmic}
\end{algorithm}

\subsection{Quantitative Analysis on TD($\lambda$) Target Distribution}
\label{sec:toy example of target}

We experiment with a toy example to measure the bias and variance of the reward estimation according to $\lambda$. 
The toy example has two states, $s_1$ and $s_2$; the state distribution is defined as an uniform; the reward function is defined as $r(s_1) \sim \mathcal{N}(-0.005, 0.02)$ and $r(s_2) \sim \mathcal{N}(0.005, 0.03)$. 
We train parameterized reward distributions by minimizing the quantile regression loss with the TD($\lambda$) target distribution for $\lambda = 0, 0.5, 0.9$, and $1.0$.
The experimental results are presented in the table below.

\begin{table}[ht]
\caption{Experimental results of the toy example.}
\label{table:results of toy example}
\vskip 0.15in
\small
\begin{center}
\begin{tabular}{llllll}
\toprule
& 5th iteration & 10th iteration & 15th iteration & 20th iteration & 25th iteration \\
\midrule
$\lambda=0.0$ & 4.813 (0.173) & 4.024 (0.253)  & 3.498 (0.085)  & 3.131 (0.103)  & 2.835 (0.070)  \\
$\lambda=0.5$ & 4.621 (0.185) & 3.688 (0.273)  & 2.925 (0.183)  & 2.379 (0.134)  & 2.057 (0.070)  \\
$\lambda=0.9$ & 4.141 (0.461) & 2.237 (0.402)  & 1.389 (0.132)  & 1.058 (0.031)  & 0.923 (0.019)  \\
$\lambda=1.0$ & 2.886 (0.767) & 1.733 (0.365)  & 1.509 (0.514)  & 1.142 (0.325)  & 1.109 (0.476) \\
\bottomrule
\end{tabular}
\end{center}
\end{table}

The values in the table are the mean and standard deviation of the past five values of the Wasserstein distance between the true reward return and the estimated distribution. 
Looking at the fifth iteration, it is clear that the larger the $\lambda$ value, the smaller the mean and the higher the standard deviation.
At the 25th iteration, the run with $\lambda=0.9$ has the lowest mean and standard deviation, indicating that training has converged.
On the other hand, the run with $\lambda=1.0$ has the biggest standard deviation, and the mean is greater than $\lambda=0.9$, indicating that the significant variance hinders training. 
In conclusion, we measured bias and variance quantitatively through the toy example, and the results are well aligned with our claim that $\lambda$ can trade off bias and variance.

\subsection{Surrogate Functions}
\label{sec:surrogate}

In this section, we introduce the surrogate functions for the objective and constraints.
First, \citet{kim2022offtrc} define a doubly discounted state distribution: $d_2^{\pi}(s):=(1-\gamma^2)\sum_{t=0}^\infty \gamma^{2t}\mathrm{Pr}(s_t=s|\pi)$.
Then, the surrogates for the objective and constraints are defined as follows \citep{kim2022offtrc}:
\begin{equation}
\label{eq:original surrogate}
\small
\begin{aligned}
J^{\mu, \pi}(\pi') &:= \underset{s_0\sim\rho}{\mathbb{E}}\left[V^\pi(s_0)\right] + \frac{1}{1-\gamma}\Big(\beta\underset{d^\pi}{\mathbb{E}}\left[H(\pi'(\cdot|s))\right] + \underset{d^\mu, \pi'}{\mathbb{E}}\left[Q_R^\pi(s,a)\right]\Big), \\
J_{C_k}^{\mu, \pi}(\pi') &:= \underset{s_0\sim\rho}{\mathbb{E}}\left[V_{C_k}^\pi(s_0)\right] + \frac{1}{1-\gamma}\underset{d^\mu, \pi'}{\mathbb{E}}\left[Q_{C_k}^\pi(s,a)\right], \\
J^{\mu,\pi}_{S_k}(\pi') &:= \underset{s_0\sim\rho}{\mathbb{E}}\left[S_{C_k}^\pi(s_0)\right] + \frac{1}{1-\gamma^2}\underset{d_2^\mu, \pi'}{\mathbb{E}}\left[S_{C_k}^\pi(s,a)\right], \\
F_k^{\mu, \pi}(\pi';\alpha) &:= J^{\mu,\pi}_{C_k}(\pi') + \frac{\phi(\Phi^{-1}(\alpha))}{\alpha}\sqrt{J^{\mu,\pi}_{S_k}(\pi') - (J^{\mu,\pi}_{C_k}(\pi'))^2},
\end{aligned}
\end{equation}
where $\mu, \pi, \pi'$ are behavioral, current, and next policies, respectively. 
According to Theorem 1 in \citep{kim2022offtrc}, the constraint surrogates are bounded by $D_\mathrm{KL}(\pi, \pi')$.
We also show that the surrogate of the objective is bounded by $D_\mathrm{KL}(\pi, \pi')$ in Appendix \ref{sec:bound of entropy-augmented objective}.
As a result, the gradients of the objective function and constraints become the same as the gradients of the surrogates, and the surrogates can substitute the objective and constraints within the trust region.

\subsection{Policy Update Rule}
\label{sec:policy update rule}

To solve the constrained optimization problem (\ref{eq:safe RL subproblem}), we find a policy update direction by linearly approximating the objective and safety constraints and quadratically approximating the trust region constraint, as done by \citet{achiam2017cpo}.
After finding the direction, we update the policy using a line search method.
Given the current policy parameter $\psi_t \in \Psi$, the approximated problem can be expressed as follows:
\begin{equation}
\label{eq:LQCLP}
x^* = \underset{x \in \Psi}{\mathrm{argmax}} \; g^{T}x \quad \mathrm{s.t.} \;
\frac{1}{2}x^T H x \leq \epsilon, \; b_k^T x + c_k \leq 0 \; \forall k,
\end{equation}
where $g = \nabla_\psi J^{\mu,\pi}(\pi_\psi)|_{\psi=\psi_t}$, $H=\nabla_{\psi}^2D_{\mathrm{KL}}(\pi_{\psi_t}||\pi_{\psi})|_{\psi=\psi_t}$, $b_k = \nabla_\psi F_k^{\mu,\pi}(\pi_\psi;\alpha)|_{\psi=\psi_t}$, and $c_k = F_k(\pi_\psi;\alpha) - d_k$.
Since (\ref{eq:LQCLP}) is convex, we can use an existing convex optimization solver.
However, the search space, which is the policy parameter space $\Psi$, is excessively large, so we reduce the space by converting (\ref{eq:LQCLP}) to a dual problem as follows:
\begin{equation}
\begin{aligned}
g(\lambda, \nu) &= \mathrm{min}_x L(x,\lambda, \nu) =\mathrm{min}_x\{-g^{T}x + \nu(\frac{1}{2}x^T H x - \epsilon) + \lambda^T(B x + c)\} \\
&=\frac{-1}{2\nu}\left(\underbrace{g^TH^{-1}g}_{=:q} - 2\underbrace{g^TH^{-1}B^T}_{=:r^T}\lambda + \lambda^T\underbrace{B H^{-1}B^T}_{=:S}\lambda\right) + \lambda^T c - \nu\epsilon \\
&= \frac{-1}{2\nu}(q - 2r^T\lambda + \lambda^T S \lambda) + \lambda^T c - \nu\epsilon,
\end{aligned}
\end{equation}
where $B = (b_1, .., b_K)$, $c = (c_1, ..., c_K)^T$, and $\lambda \in \mathbb{R}^K \geq 0$ and $\nu \in \mathbb{R} \geq 0$ are Lagrange multipliers.
Then, the optimal $\lambda$ and $\nu$ can be obtained by a convex optimization solver.
After obtaining the optimal values, $(\lambda^*, \nu^*) = \mathrm{argmax}_{(\lambda, \nu)}g(\lambda, \nu)$, the policy update direction $x^*$ are calculated by $\frac{1}{\nu^*}H^{-1}(g - B^T \lambda^*)$.
Then, the policy is updated by $\psi_{t+1} = \psi_t + \beta x^*$, where $\beta$ is a step size, which can be found through a backtracking method (please refer to Section 6.3.2 of \cite{dennis1996numerical}).

Before using the above policy update rule, we should note that the existing trust-region method with the risk-averse constraint \citep{kim2022offtrc} and the equations (\ref{eq:safe RL problem}, \ref{eq:original surrogate}, \ref{eq:safe RL subproblem}) are slightly different.
There are two differences: 1) the objective is augmented with an entropy bonus, and 2) the surrogates are expressed with Q-functions instead of value functions.
To use the entropy-regularized objective in the trust-region method, it is required to show that the objective is bounded by the KL divergence.
We present the existence of bound in Appendix \ref{sec:bound of entropy-augmented objective}.
Next, there is no problem using the Q-functions because it is mathematically equivalent between the original surrogates \citep{kim2022offtrc} and the new ones expressed with Q-functions (\ref{eq:original surrogate}). 
However, we experimentally show that using the Q-functions in off-policy settings has advantages in Appendix \ref{sec:compare surrogate}.

\subsection{Bound of Entropy-Augmented Objective}
\label{sec:bound of entropy-augmented objective}

In this section, we show that the entropy-regularzied objective in (\ref{eq:safe RL problem}) has a bound expressed by the KL divergence.
Before showing the boundness, we present a new function and a lemma.
A value difference function is defined as follows:
\begin{equation*}
\delta^{\pi'}(s):=\mathbb{E}\left[R(s,a,s')+\gamma V^{\pi}(s') - V^\pi(s) \; | \; a\sim\pi'(\cdot|s), s'\sim P(\cdot|s,a)\right] = \underset{a\sim\pi'}{\mathbb{E}}\left[A^\pi(s,a)\right],
\end{equation*}
where $A^\pi(s,a) := Q^\pi(s,a) - V^\pi(s,a)$.
\begin{lemma}
\label{lemma:delta difference}
The maximum of $|\delta^{\pi'}(s) - \delta^{\pi}(s)|$ is equal or less than $\epsilon_R\sqrt{2D_{\mathrm{KL}}^{\mathrm{max}}(\pi||\pi')}$, where $\epsilon_R=\underset{s,a}{\mathrm{max}}\left|A^{\pi}(s,a)\right|$.
\end{lemma}
\begin{proof}
The value difference can be expressed in a vector form,
\begin{equation*}
\delta^{\pi'}(s) - \delta^{\pi}(s) = \sum_a(\pi'(a|s) - \pi(a|s))A^\pi(s, a) = \langle \pi'(\cdot|s) - \pi(\cdot|s), A^{\pi}(s, \cdot) \rangle.
\end{equation*}
Using Hölder's inequality, the following inequality holds:
\begin{equation*}
\begin{aligned}
|\delta^{\pi'}(s) - \delta^{\pi}(s)| &\leq ||\pi'(\cdot|s) - \pi(\cdot|s)||_1 \cdot ||A^\pi(s,\cdot)||_{\infty} \\
&=2D_{\mathrm{TV}}(\pi'(\cdot|s)||\pi(\cdot|s))\mathrm{max}_a A^\pi(s,a).
\end{aligned}
\end{equation*}
\begin{equation*}
\Rightarrow ||\delta^{\pi'} - \delta^{\pi}||_{\infty} = \mathrm{max}_s |\delta^{\pi'}(s) - \delta^{\pi}(s)| \leq 2\epsilon_R\mathrm{max}_s D_{\mathrm{TV}}(\pi(\cdot|s)||\pi'(\cdot|s)).
\end{equation*}
Using Pinsker's inequality,
$||\delta^{\pi'} - \delta^{\pi}||_{\infty} \leq \epsilon_R\sqrt{2D_{\mathrm{KL}}^{\mathrm{max}}(\pi||\pi')}$.
\end{proof}

\begin{theorem}
\label{thm:bound}
Let us assume that $\mathrm{max}_s H(\pi(\cdot|s)) < \infty$ for $\forall\pi\in\Pi$.
The difference between the objective and surrogate functions is bounded by a term consisting of KL divergence as:
\begin{equation}
\small
\begin{aligned}
&\left|J(\pi') - J^{\mu, \pi}(\pi') \right| \leq \frac{\sqrt{2}\gamma}{(1 - \gamma)^2}\sqrt{D_{\mathrm{KL}}^{\mathrm{max}}(\pi||\pi')}\left(\beta\epsilon_H + \epsilon_R\sqrt{2D_{\mathrm{KL}}^{\mathrm{max}}(\mu||\pi')}\right),
\end{aligned}
\end{equation}
where $\epsilon_H=\underset{s}{\mathrm{max}}\left|H(\pi'(\cdot|s))\right|$, $D_{\mathrm{KL}}^{\mathrm{max}}(\pi||\pi')=\underset{s}{\mathrm{max}}\;D_{\mathrm{KL}}(\pi(\cdot|s)||\pi'(\cdot|s))$, and the equality holds when $\pi'=\pi$.
\end{theorem}

\begin{proof}
The surrogate function can be expressed in vector form as follows:
\begin{equation*}
J^{\mu, \pi}(\pi') = \langle\rho, V^\pi\rangle + \frac{1}{1-\gamma}\left(\langle d^\mu, \delta^{\pi'}\rangle + \beta\langle d^\pi, H^{\pi'}\rangle\right),
\end{equation*}
where $H^{\pi'}(s)=H(\pi'(\cdot|s))$.
The objective function of $\pi'$ can also be expressed in a vector form using Lemma 1 from \citet{achiam2017cpo},
\begin{equation*}
\begin{aligned}
J(\pi') &= \frac{1}{1-\gamma}\mathbb{E}\left[R(s,a,s') + \beta H^{\pi'}(s) \; | \; s\sim d^{\pi'}, a\sim\pi'(\cdot|s), s'\sim P(\cdot|s,a)\right] \\
&=\frac{1}{1-\gamma}\underset{s\sim d^{\pi'}}{\mathbb{E}}\left[\delta^{\pi'}(s) + \beta H^{\pi'}(s)\right] + \underset{s\sim\rho}{\mathbb{E}}\left[V^\pi(s)\right] \\
&=\langle\rho, V^\pi\rangle + \frac{1}{1-\gamma}\langle d^{\pi'}, \delta^{\pi'} + \beta H^{\pi'}\rangle.
\end{aligned}
\end{equation*}
By Lemma 3 from \citet{achiam2017cpo}, $||d^\pi - d^{\pi'}||_1 \leq \frac{\gamma}{1 - \gamma}\sqrt{2D^{\mathrm{max}}_{\mathrm{KL}}(\pi||\pi')}$.
Then, the following inequality is satisfied:
\begin{equation*}
\begin{aligned}
|(1-&\gamma)(J^{\mu, \pi}(\pi') - J(\pi'))| \\
&= |\langle d^{\pi'}-d^\mu, \delta^{\pi'}\rangle + \beta\langle d^{\pi} - d^{\pi'}, H^{\pi'}\rangle|\\
&\leq |\langle d^{\pi'}-d^\mu, \delta^{\pi'}\rangle| + \beta |\langle d^{\pi} - d^{\pi'}, H^{\pi'}\rangle|\\
&= |\langle d^{\pi'}-d^\mu, \delta^{\pi'} - \delta^{\pi}\rangle| + \beta |\langle d^{\pi} - d^{\pi'}, H^{\pi'}\rangle| && (\because \delta^{\pi}=0)\\
&\leq || d^{\pi'}-d^\mu||_1 ||\delta^{\pi'} - \delta^{\pi}||_{\infty} + \beta || d^{\pi} - d^{\pi'}||_1 ||H^{\pi'}||_\infty && (\because \text{\small Hölder's inequality})\\
&\leq \frac{2\epsilon_R \gamma}{1-\gamma}\sqrt{D_{\mathrm{KL}}^{\mathrm{max}}(\mu||\pi')D_{\mathrm{KL}}^{\mathrm{max}}(\pi||\pi')} + \frac{\beta\gamma\epsilon_H}{1-\gamma}\sqrt{2D^{\mathrm{max}}_{\mathrm{KL}}(\pi||\pi')} && (\because \text{Lemma \ref{lemma:delta difference}})\\
&= \frac{\gamma}{1 - \gamma}\sqrt{D_{\mathrm{KL}}^{\mathrm{max}}(\pi||\pi')}\left(\sqrt{2}\beta\epsilon_H + 2\epsilon_R\sqrt{D_{\mathrm{KL}}^{\mathrm{max}}(\mu||\pi')}\right).
\end{aligned}    
\end{equation*}
If $\pi'=\pi$, the KL divergence term becomes zero, so equality holds.
\end{proof}

\subsection{Comparison of Q-Function and Value Function-Based Surrogates}
\label{sec:compare surrogate}

The original surrogate is defined as follows:
\begin{equation}
\label{eq:off-trpo surrogates}
J^{\mu, \pi}(\pi') := J(\pi) + \frac{1}{1-\gamma}\underset{d^\mu, \mu}{\mathbb{E}}\left[\frac{\pi'(a|s)}{\mu(a|s)}A^\pi(s,a)\right],
\end{equation}
where $A^\pi(s,a) := Q^\pi(s,a) - V^\pi(s,a)$, and the surrogate is the same as that of OffTRPO \citep{meng2022offtrpo} and OffTRC \citep{kim2022offtrc}.
An entropy-regularized version can be derived as:
\begin{equation}
\label{eq:entropy-regularized surrogates}
J^{\mu, \pi}(\pi') = J(\pi) + \frac{1}{1-\gamma}\Big(\beta\underset{d^\pi}{\mathbb{E}}\left[H(\pi'(\cdot|s))\right] + \underset{d^\mu, \mu}{\mathbb{E}}\left[\frac{\pi'(a|s)}{\mu(a|s)}A^\pi(s,a)\right]\Big).
\end{equation}
Then, the surrogate expressed by Q-functions in (\ref{eq:original surrogate}), called SAC-style version, can be rewritten as:
\begin{equation}
\label{eq:Q-function-based surrogates}
J^{\mu, \pi}(\pi') = J(\pi) + \frac{1}{1-\gamma}\Big(\beta\underset{d^\pi}{\mathbb{E}}\left[H(\pi'(\cdot|s))\right] + \underset{d^{\mu},\pi'}{\mathbb{E}}\left[Q^{\pi}(s, a)\right] \Big).
\end{equation}
In this section, we evaluate the original, entropy-regularized, and SAC-style versions in the continuous control tasks of the MuJoCo simulators \citep{todorov2012mujoco}.
We use neural networks with two hidden layers with (512, 512) nodes and ReLU for the activation function.
The output of a value network is linear, but the input is different; the original and entropy-regularized versions use states, and the SAC-style version uses state-action pairs.
The input of a policy network is the state, the output is mean $\mu$ and std $\sigma$, and actions are squashed into $\mathrm{tanh}(\mu + \epsilon\sigma)$, $\epsilon \sim \mathcal{N}(0,1)$ as in SAC \citep{haarnoja2018sac}.
The entropy coefficient $\beta$ in the entropy-regularized and SAC-style versions are adaptively adjusted to keep the entropy above a threshold (set as $-d$ given $A \subseteq \mathbb{R}^d$).
The hyperparameters for all versions are summarized in Table \ref{table:surrogate settings}.
\begin{table}[ht]
\caption{Hyperparameters for all versions.}
\label{table:surrogate settings}
\vskip 0.15in
\small
\begin{center}
\begin{tabular}{l|l}
\toprule
Parameter & Value \\
\midrule
Discount factor $\gamma$ & 0.99 \\
Trust region size $\epsilon$ & 0.001 \\
Length of replay buffer & $10^5$ \\
Critic learning rate & $0.0003$ \\
Trace-decay $\lambda$ & $0.97$ \\
Initial entropy coefficient $\beta$ & 1.0 \\
$\beta$ learning rate & 0.01 \\
\bottomrule
\end{tabular}
\end{center}
\end{table}

\begin{figure}[ht]
\centering
\begin{subfigure}[b]{0.3\textwidth}
    \centering
    \includegraphics[width=\textwidth]{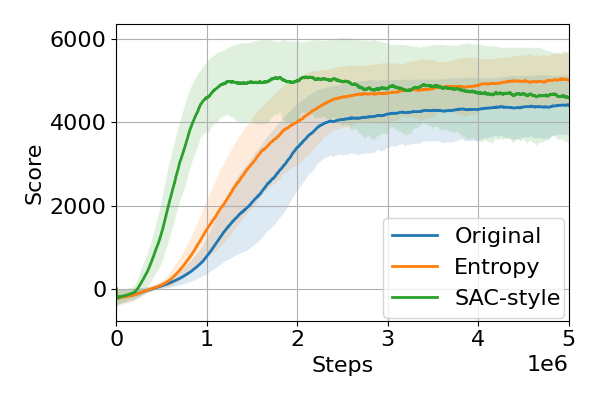}
    \caption{Ant-v3}
\end{subfigure}
\begin{subfigure}[b]{0.3\textwidth}
    \centering
    \includegraphics[width=\textwidth]{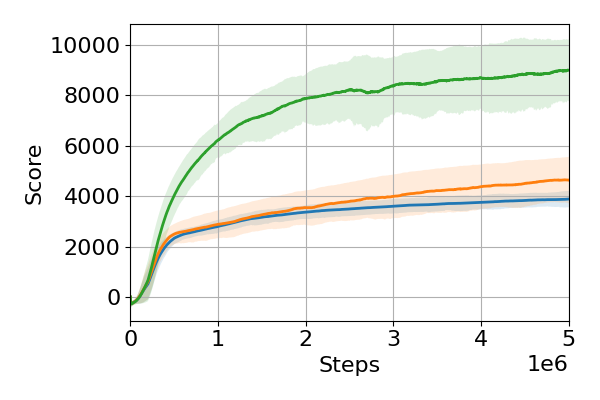}
    \caption{HalfCheetah-v3}
\end{subfigure}
\begin{subfigure}[b]{0.3\textwidth}
    \centering
    \includegraphics[width=\textwidth]{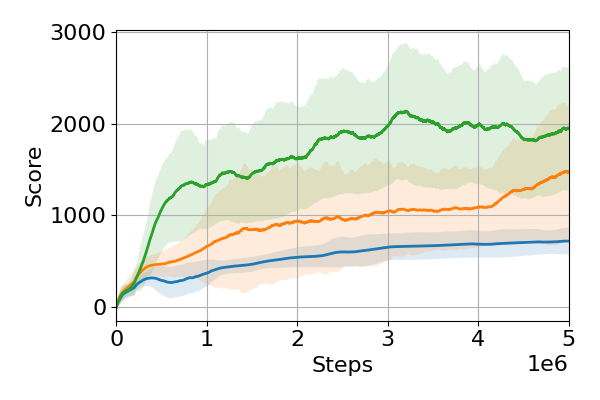}
    \caption{Hopper-v3}
\end{subfigure}
\begin{subfigure}[b]{0.3\textwidth}
    \centering
    \includegraphics[width=\textwidth]{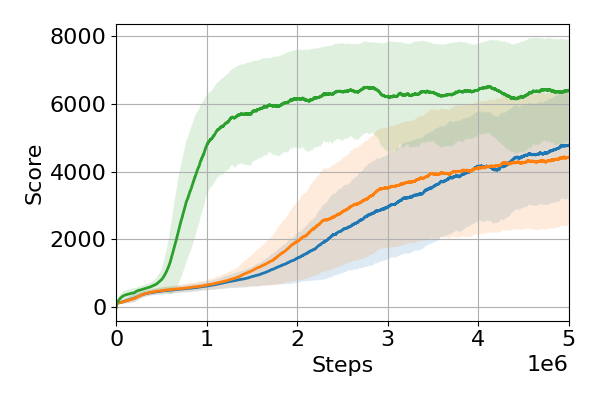}
    \caption{Humanoid-v3}
\end{subfigure}
\begin{subfigure}[b]{0.3\textwidth}
    \centering
    \includegraphics[width=\textwidth]{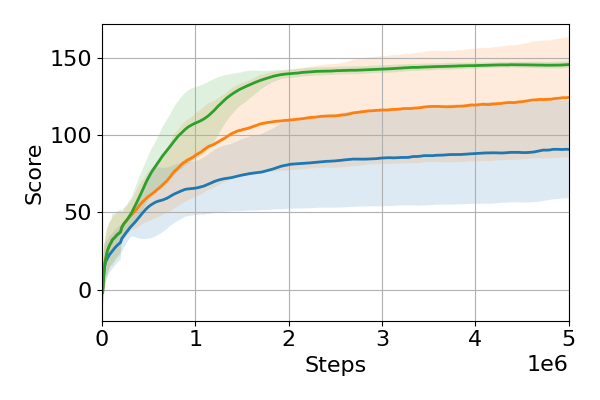}
    \caption{Swimmer-v3}
\end{subfigure}
\begin{subfigure}[b]{0.3\textwidth}
    \centering
    \includegraphics[width=\textwidth]{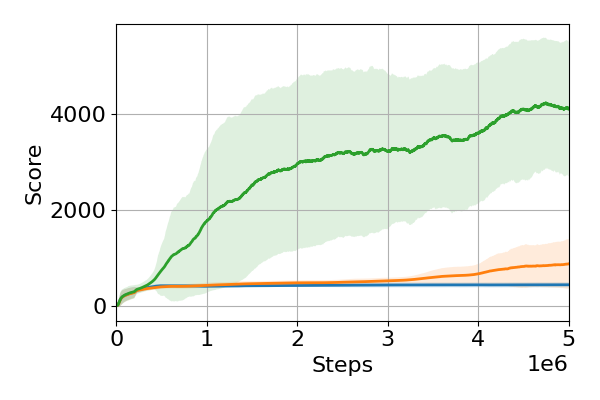}
    \caption{Walker2d-v3}
\end{subfigure}
\caption{MuJoCo training curves.}
\label{fig:MuJoCo Training Curves}
\end{figure}

The training curves are presented in Figure \ref{fig:MuJoCo Training Curves}.
All methods are trained with five different random seeds.
Although the entropy-regularized version (\ref{eq:entropy-regularized surrogates}) and SAC-style version (\ref{eq:Q-function-based surrogates}) are mathematically equivalent, it can be observed that the performance of the SAC-style version is superior to the regularized version.
It can be inferred that this is due to the variance of importance sampling.
In the off-policy setting, the sampling probabilities of the behavioral and current policies can be significantly different, so the variance of the importance ratio is huge.
The increased variance prevents estimating the objective accurately, so significant performance degradation can happen.
As a result, using the Q-function-based surrogates has an advantage for efficient learning.

\newpage
\section{Experimental Settings}
\label{sec:experimental settings}

\begin{figure}[h]
\centering
\begin{subfigure}[b]{0.19\textwidth}
    \centering
    \includegraphics[width=\textwidth]{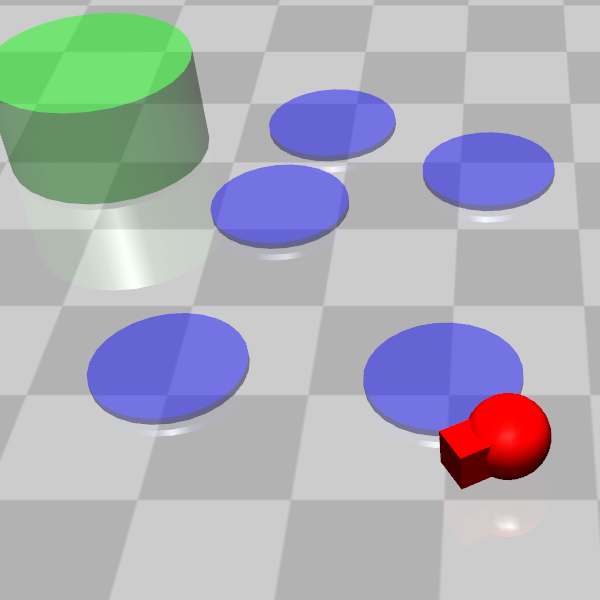}
    \caption{Point goal.}
    \label{sfig:point goal task}    
\end{subfigure}
\begin{subfigure}[b]{0.19\textwidth}
    \centering
    \includegraphics[width=\textwidth]{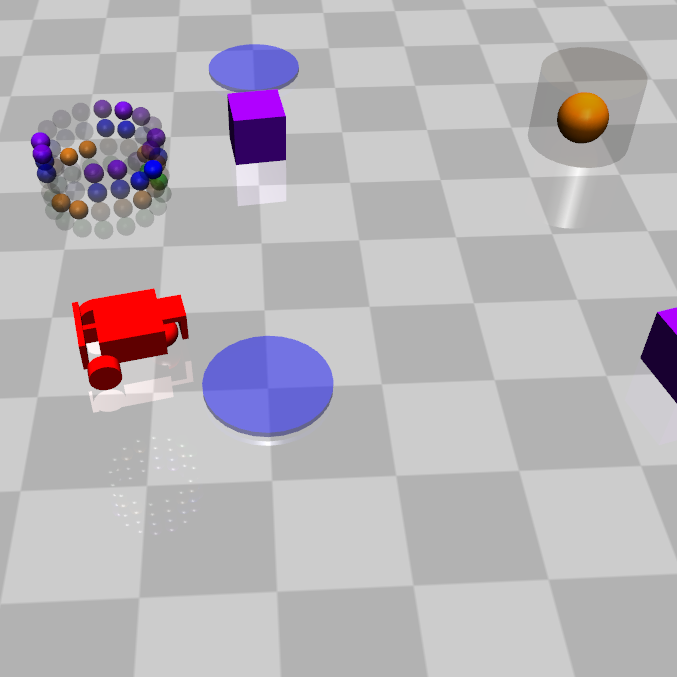}
    \caption{Car button.}
    \label{sfig:car button task}    
\end{subfigure}
\begin{subfigure}[b]{0.19\textwidth}
    \centering
    \includegraphics[width=\textwidth]{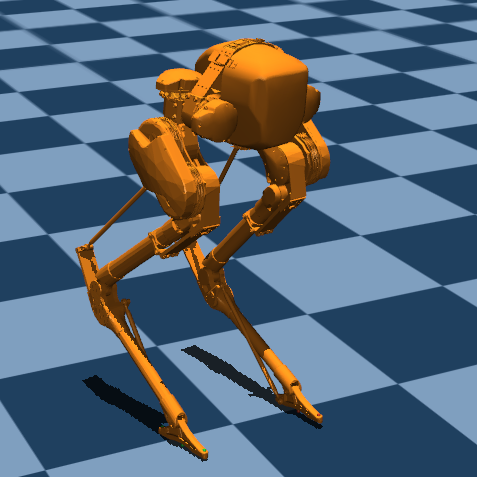}
    \caption{Cassie.}
    \label{sfig:cassie task}
\end{subfigure}
\begin{subfigure}[b]{0.19\textwidth}
    \centering
    \includegraphics[width=\textwidth]{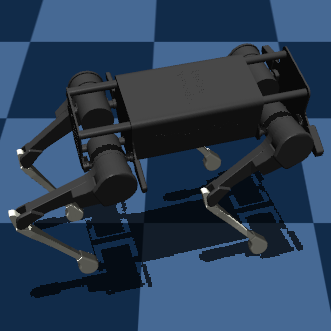}
    \caption{Laikago.}
    \label{sfig:laikago task}
\end{subfigure}
\begin{subfigure}[b]{0.19\textwidth}
    \centering
    \includegraphics[width=\textwidth]{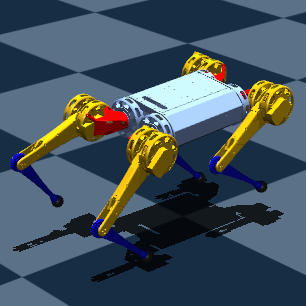}
    \caption{Mini-Cheetah.}
    \label{sfig:mini-cheetah task}
\end{subfigure}
\caption{(a) and (b) are Safety Gym tasks. (c), (d), and (e) are locomotion tasks.}
\label{fig:tasks}
\end{figure}

\textbf{Safety Gym.}
We use the goal and button tasks with the point and car robots in the Safety Gym environment \citep{ray2019safetygym}, as shown in Figure \ref{sfig:point goal task} and \ref{sfig:car button task}.
The environmental setting for the goal task is the same as in  \citet{kim2022trc}.
Eight hazard regions and one goal are randomly spawned at the beginning of each episode, and a robot gets a reward and cost as follows:
\begin{equation}
\label{eq:goal task reward}
\begin{aligned}
R(s,a,s') &= -\Delta d_\mathrm{goal} + \mathbf{1}_{d_{\mathrm{goal}} \leq 0.3}, \\
C(s,a,s') &= \mathrm{Sigmoid}(10\cdot(0.2 - d_\mathrm{hazard})),
\end{aligned}
\end{equation}
where $d_\mathrm{goal}$ is the distance to the goal, and $d_\mathrm{hazard}$ is the minimum distance to hazard regions.
If $d_\mathrm{goal}$ is less than or equal to $0.3$, a goal is respawned.
The state consists of relative goal position, goal distance, linear and angular velocities, acceleration, and LiDAR values.
The action space is two-dimensional, which consists of $xy$-directional forces for the point and wheel velocities for the car robot. \par
The environmental settings for the button task are the same as in \citet{liu2022cvpo}.
There are five hazard regions, four dynamic obstacles, and four buttons, and all components are fixed throughout the training.
The initial position of a robot and an activated button are randomly placed at the beginning of each episode.
The reward function is the same as in (\ref{eq:goal task reward}), but the cost is different since there is no dense signal for contacts.
We define the cost function for the button task as an indicator function that outputs one if the robot makes contact with an obstacle or an inactive button or enters a hazardous region.
We add LiDAR values of buttons and obstacles to the state of the goal task, and actions are the same as the goal task.
The length of the episode is 1000 steps without early termination.

\textbf{Locomotion Tasks.}
We use three different legged robots, Mini-Cheetah, Laikago, and Cassie, for the locomotion tasks, as shown in Figure \ref{sfig:mini-cheetah task}, \ref{sfig:laikago task}, and \ref{sfig:cassie task}.
The tasks aim to control robots to follow a velocity command on flat terrain.
A velocity command is given by $(v_x^\mathrm{cmd}, v_y^\mathrm{cmd}, \omega_z^\mathrm{cmd})$, where $v_x^\mathrm{cmd} \sim \mathcal{U}(-1.0, 1.0)$ for Cassie and $\mathcal{U}(-1.0, 2.0)$ otherwise, $v_y^\mathrm{cmd}=0$, and $\omega_z^\mathrm{cmd}\sim\mathcal{U}(-0.5, 0.5)$.
To lower the task complexity, we set the $y$-directional linear velocity to zero but can scale to any non-zero value.
As in other locomotion studies \citep{lee2020challenging, miki2022wild}, \emph{central phases} are introduced to produce periodic motion, which are defined as $\phi_i(t) = \phi_{i, 0} + f\cdot t$ for $\forall i \in \{1, ..., n_\mathrm{legs}\}$, where $f$ is a frequency coefficient and is set to $10$, and $\phi_{i, 0}$ is an initial phase. 
Actuators of robots are controlled by PD control towards target positions given by actions.
The state consists of velocity command, orientation of the robot frame, linear and angular velocities of the robot, positions and speeds of the actuators, central phases, history of positions and speeds of the actuators (past two steps), and history of actions (past two steps).
A foot contact timing $\xi$ can be defined as follows:
\begin{equation}
\xi_i(s) = -1 + 2\cdot \mathbf{1}_{\sin(\phi_i) \leq 0} \;\; \forall i \in \{1, ..., n_\mathrm{legs}\},
\end{equation}
where a value of -1 means that the $i$th foot is on the ground; otherwise, the foot is in the air.
For the quadrupedal robots, Mini-Cheetah and Laikago, we use the initial phases as $\phi_0 = \{0, \pi, \pi, 0\}$, which generates trot gaits.
For the bipedal robot, Cassie, the initial phases are defined as $\phi_0 = \{0, \pi\}$, which generates walk gaits.
Then, the reward and cost functions are defined as follows:
\begin{equation}
\label{eq:reward and costs}
\small
\begin{aligned}
&\qquad\qquad R(s,a,s') = -0.1\cdot(||v_{x,y}^{\mathrm{base}} - v_{x,y}^{\mathrm{cmd}}||_2^2 + ||\omega_{z}^{\mathrm{base}} - \omega_{z}^{\mathrm{cmd}}||_2^2 + 10^{-3}\cdot R_\mathrm{power}), \\
&C_1(s, a, s') = \mathbf{1}_{\mathrm{angle}\geq a}, \; C_2(s, a, s') = \mathbf{1}_{\mathrm{height}\leq b}, \; C_3(s, a, s') = \sum_{i=1}^{n_{\mathrm{legs}}}(1 - \xi_i\cdot \hat{\xi}_i)/(2\cdot n_{\mathrm{legs}}),
\end{aligned}
\end{equation}
where the power consumption $R_{\mathrm{power}}=\sum_i |\tau_i v_i|$, the sum of the torque times the actuator speed, is added to the reward as a regularization term, $v_{x,y}^{\mathrm{base}}$ is the $xy$-directional linear velocity of the base frame of robots, $\omega_z^{\mathrm{base}}$ is the $z$-directional angular velocity of the base frame, and $\hat{\xi} \in \{-1, 1\}^{n_\mathrm{legs}}$ is the current feet contact vector.
For balancing, the first cost indicates whether the angle between the $z$-axis vector of the robot base and the world is greater than a threshold ($a=15^{\circ}$ for all robots).
For standing, the second cost indicates the height of CoM is less than a threshold ($b=0.3, 0.35, 0.7$ for Mini-Cheetah, Laikago, and Cassie, respectively), and the last cost is to check that the current feet contact vector $\hat{\xi}$ matches the pre-defined timing $\xi$.
The length of the episode is 500 steps.
There is no early termination, but if a robot falls to the ground, the state is frozen until the end of the episode.

\textbf{Hyperparameter Settings.}
The structure of neural networks consists of two hidden layers with $(512, 512)$ nodes and ReLU activation for all baselines and the proposed method.
The input of value networks is state-action pairs, and the output is the positions of atoms.
The input of policy networks is the state, the output is mean $\mu$ and std $\sigma$, and actions are squashed into $\mathrm{tanh}(\mu + \epsilon\sigma)$, $\epsilon \sim \mathcal{N}(0,1)$.
We use a fixed entropy coefficient $\beta$.
The trust region size $\epsilon$ is set to $0.001$ for all trust region-based methods.
The overall hyperparameters for the proposed method can be summarized in Table \ref{table: safe RL settings}.
\begin{table}[ht]
\caption{Hyperparameter settings for the Safety Gym and locomotion tasks.}
\label{table: safe RL settings}
\vskip 0.15in
\begin{center}
\begin{small}
\begin{tabular}{l|l|l}
\toprule
Parameter & Safety Gym & Locomotion \\ 
\midrule
Discount factor $\gamma$ & 0.99 & 0.99 \\
Trust region size $\epsilon$ & 0.001 & 0.001 \\
Length of replay buffer & $10^5$ & $10^5$ \\
Critic learning rate & $0.0003$ & $0.0003$ \\
Trace-decay $\lambda$ & $0.97$ & $0.97$ \\
Entropy coefficient $\beta$ & 0.0 & 0.001 \\
The number of critic atoms $M$ & 25 & 25 \\
The number of target atoms $M'$ & 50 & 50 \\
Constraint risk level $\alpha$ & 0.25, 0.5, and 1.0 & 1.0 \\
threshold $d_k$ & $0.025/(1 - \gamma)$ & $[0.025, 0.025, 0.4]/(1 - \gamma)$ \\
Slack coefficient $\zeta$ & - & $\mathrm{min}_k d_k = 0.025/(1-\gamma)$ \\
\bottomrule
\end{tabular}
\end{small}
\end{center}
\vskip -0.1in
\end{table}

Since the range of the cost is $[0, 1]$, the maximum discounted cost sum is $1/(1 - \gamma)$.
Thus, the threshold is set by target cost rate times $1/(1 - \gamma)$.
For the locomotion tasks, the third cost in (\ref{eq:reward and costs}) is designed for foot stamping, which is not essential to safety.
Hence, we set the threshold to near the maximum (if a robot does not stamp, the cost rate becomes 0.5).
In addition, baseline safe RL methods use multiple critic networks for the cost function, such as target \citep{yang2020wcsac} or square value networks \citep{kim2022offtrc}.
To match the number of network parameters, we use two critics as an ensemble, as in \citet{kuznetsov2020tqc}.

\textbf{Tips for Hyperparameter Tuning.}
\begin{itemize}
    \item Discount factor $\gamma$, Critic learning rate: Since these are commonly used hyperparameters, we do not discuss these.
    \item Trace-decay $\lambda$, Trust region size $\epsilon$: 
    The ablation studies on these hyperparameters are presented in Appendix \ref{sec:trace ablation on hyperparameters}.
    From the results, we recommend setting the trace-decay to $0.95\sim0.99$ as in other TD($\lambda$)-based methods \citep{precup2000tdtrace}.
    Also, the results show that the performance is not sensitive to the trust region size.
    However, if the trust region size is too large, the approximation error increases, so it is better to set it below $0.003$.
    \item Entropy coefficient $\beta$:
    This value is fixed in our experiments, but it can be adjusted automatically as done in SAC \citep{haarnoja2018sac}.
    \item The number of atoms $M, M'$: 
    Although experiments on the number of atoms did not performed, performance is expected to increase as the number of atoms increases, as in other distributional RL methods \citep{dabney2018implicit}.
    \item Length of replay buffer:
    The effect of the length of the replay buffer can be confirmed through the experimental results from an off policy-based safe RL method \citep{kim2022offtrc}.
    According to that, the length does not impact performance unless it is too short. We recommend setting it to 10 to 100 times the collected trajectory length.
    \item Constraint risk level $\alpha$, threshold $d_k$:
    If the cost sum follows a Gaussian distribution, the mean-std constraint is identical to the CVaR constraint. Then, the probability of the worst case can be controlled by adjusting $\alpha$.
    For example, if we set $\alpha=0.125$ and $d=0.03/(1-\gamma)$, the mean-std constraint enforces the probability that the average cost is less than 0.03 during an episode greater than $95\%=\Phi(\phi(\Phi^{-1}(\alpha))/\alpha)$.
    Through this meaning, proper $\alpha$ and $d_k$ can be found.
    \item Slack coefficient $\zeta$:
    As mentioned at the end of Section \ref{sec:feasibility handling}, it is recommended to set this coefficient as large as possible. Since $d_k - \zeta$ should be positive, we recommend setting $\zeta$ to $\min_k d_k$.
\end{itemize}
In conclusion, most hyperparameters are not sensitive, so few need to be optimized.
It seems that $\alpha$ and $d_k$ need to be set based on the meaning described above. 
Additionally, if the approximation error of critics is significant, the trust region size should be set smaller.

\newpage
\section{Experimental Results}
\label{sec:experimental results}

\subsection{Safety Gym}
\label{sec:experimental results on safety gym}

In this section, we present the training curves of the Safety Gym tasks separately according to the risk level of constraints for better readability.
Figure \ref{fig: training curves of safety gym alpha=1.0} shows The training results of the risk-neutral constrained algorithms and risk-averse constrained algorithms with $\alpha=1.0$.
Figures \ref{fig: training curves of safety gym alpha=0.5} and \ref{fig: training curves of safety gym alpha=0.25} show the training results of the risk-averse constrained algorithms with $\alpha=0.25$ and $0.5$, respectively.

\begin{figure}[ht]
\centering
\includegraphics[width=\linewidth]{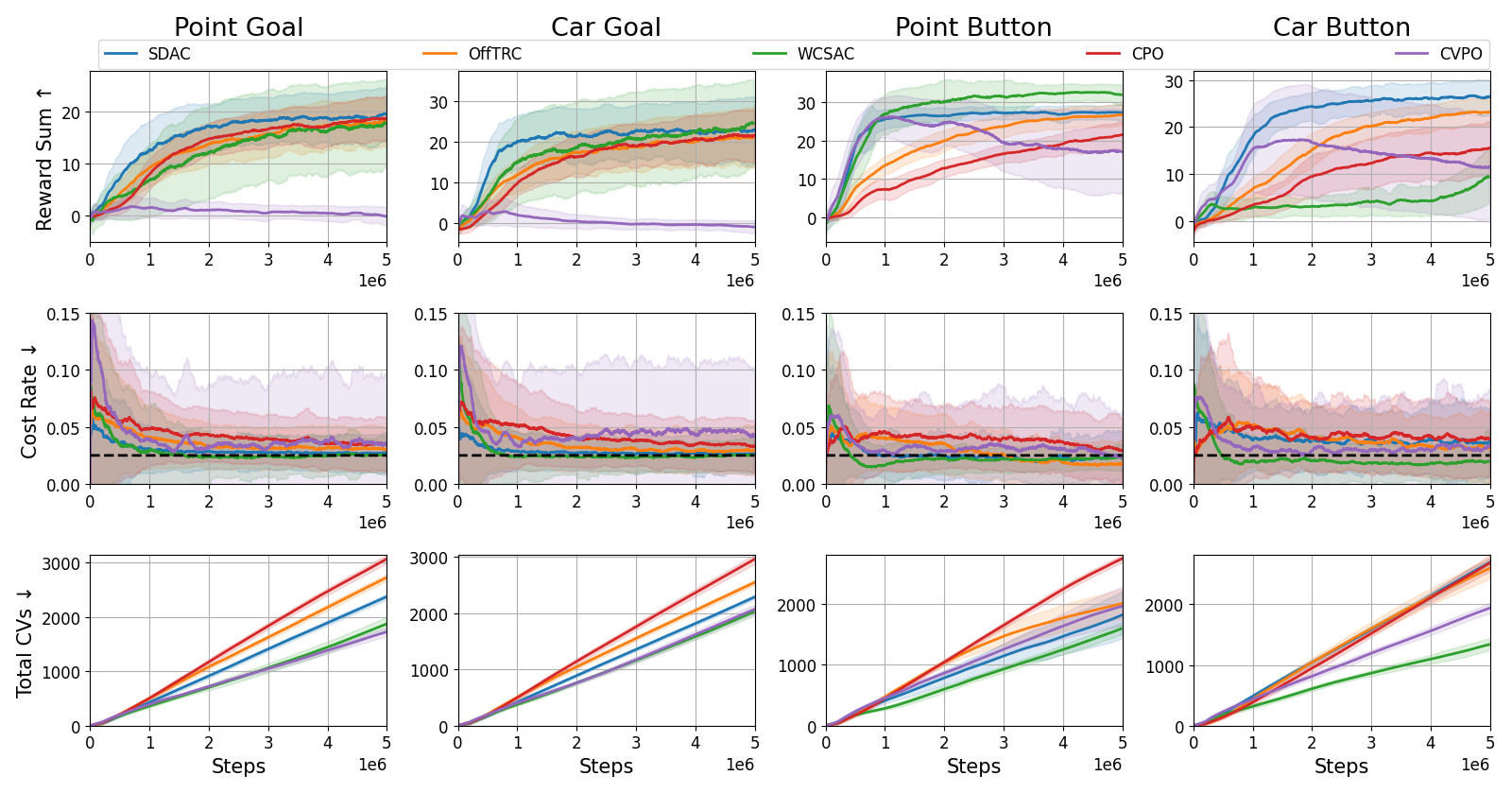}
\caption{
Training curves of risk-neutral constrained algorithms for the Safety Gym tasks.
The solid line and shaded area represent the average and std values, respectively.
The black dashed lines in the second row indicate thresholds.
All methods are trained with five random seeds.
}
\label{fig: training curves of safety gym alpha=1.0}
\end{figure}

\begin{figure}[ht]
\centering
\includegraphics[width=\linewidth]{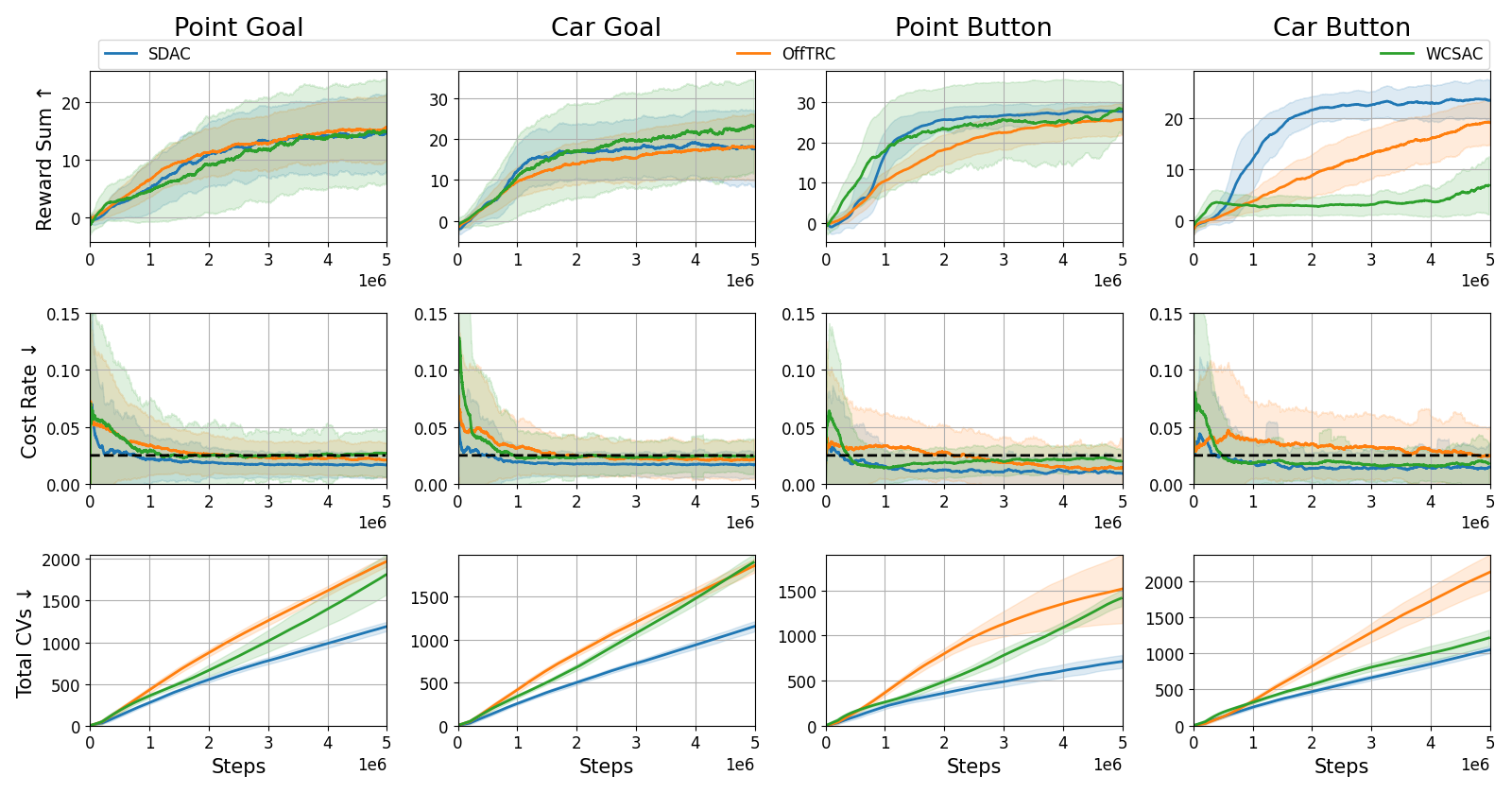}
\caption{
Training curves of risk-averse constrained algorithms with $\alpha=0.5$ for the Safety Gym.
}
\label{fig: training curves of safety gym alpha=0.5}
\end{figure}

\newpage
\begin{figure}[h]
\centering
\includegraphics[width=\linewidth]{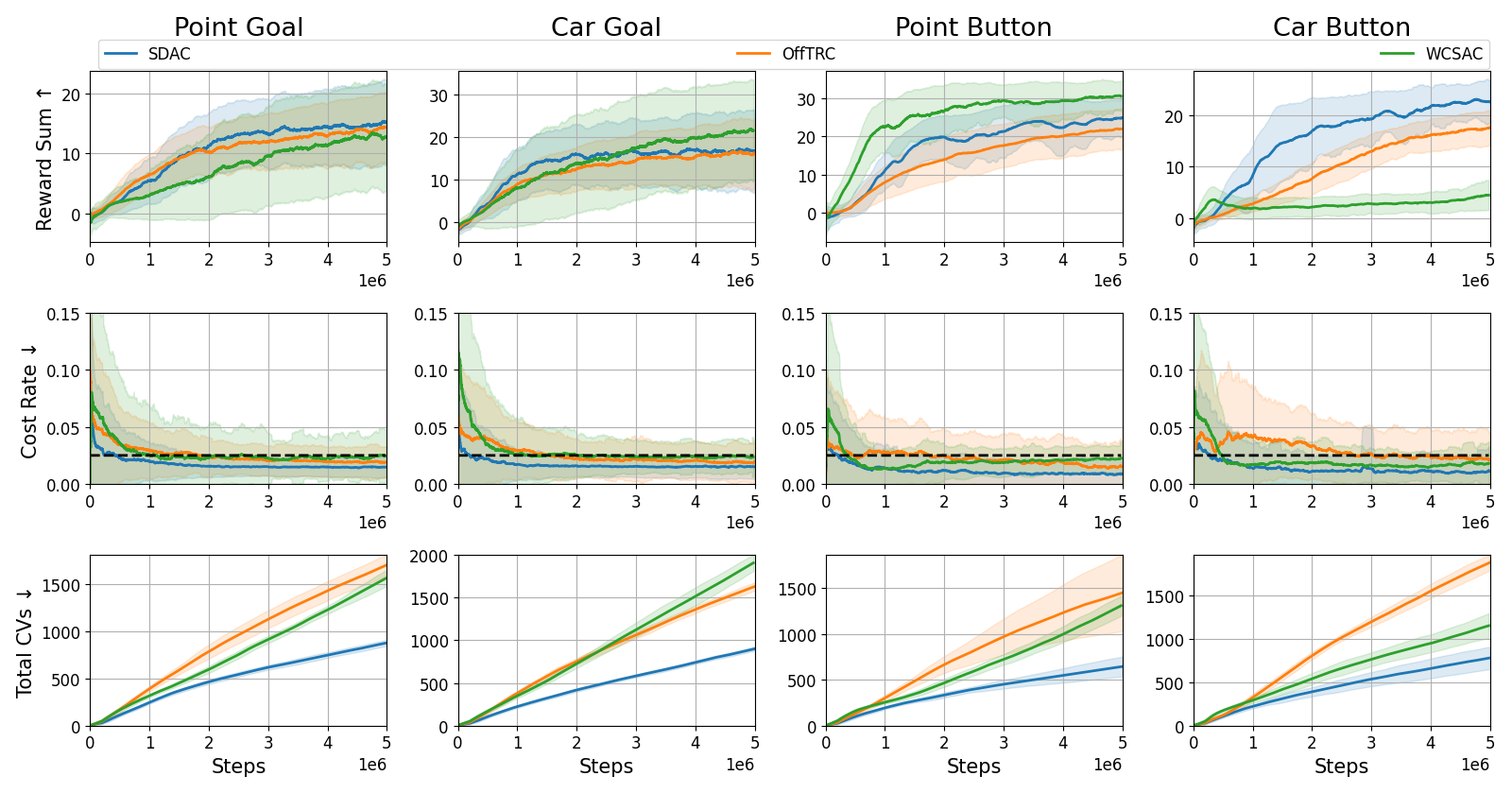}
\caption{
Training curves of risk-averse constrained algorithms with $\alpha=0.25$ for the Safety Gym.
}
\label{fig: training curves of safety gym alpha=0.25}
\end{figure}

\subsection{Ablation Study on Components of SDAC}
\label{sec: ablation on components of SDAC}

There are three main differences between SDAC and the existing trust region-based safe RL algorithm for mean-std constraints \citep{kim2022offtrc}, called OffTRC: 1) feasibility handling methods in multi-constraint settings, 2) the use of distributional critics, and 3) the use of Q-functions instead of advantage functions, as explained in Appendix \ref{sec:policy update rule} and \ref{sec:compare surrogate}.
Since the ablation study for feasibility handling is conducted in Section \ref{sec:ablation study in main text}, we perform ablation studies for the distributional critic and Q-function in this section.
We call SDAC with only distributional critics as \emph{SDAC-Dist} and SDAC with only Q-functions as \emph{SDAC-Q}.
If all components are absent, SDAC is identical to OffTRC \citep{kim2022offtrc}.
The variants are trained with the point goal task of the Safety Gym, and the training results are shown in Figure \ref{fig: compare all components}.
SDAC-Q lowers the cost rate quickly but shows the lowest score.
SDAC-Dist shows scores similar to SDAC, but the cost rate converges above the threshold $0.025$.
In conclusion, SDAC can efficiently satisfy the safety constraints through the use of Q-functions and improve score performance through the distributional critics.

\begin{figure}[ht]
\centering
\includegraphics[width=0.9\linewidth]{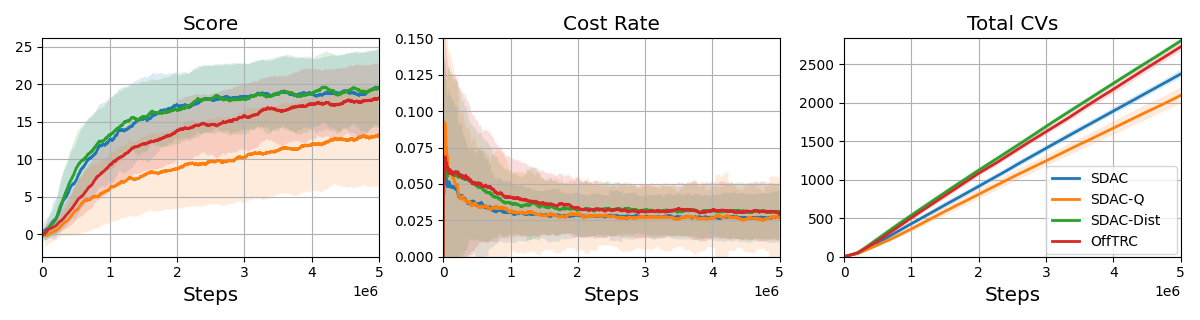}
\caption{Training curves of variants of SDAC for the point goal task.}
\label{fig: compare all components}
\end{figure}

\newpage
\subsection{Ablation Study on Hyperparameters}
\label{sec:trace ablation on hyperparameters}

To check the effects of the hyperparameters, we conduct ablation studies on the trust region size $\epsilon$ and entropy coefficient $\beta$.
The results on the entropy coefficient are presented in Figure \ref{sfig:ablation on entropy}, showing that the score significantly decreases when $\beta$ is $0.01$. 
This indicates that policies with high entropy fail to improve score performance since they focus on satisfying the constraints.
Thus, the entropy coefficient should be adjusted cautiously, or it can be better to set the coefficient to zero.
The results on the trust region size are shown in Figure \ref{sfig:ablation on trust region}, which shows that the results do not change significantly regardless of the trust region size.
However, the score convergence rate for $\epsilon=0.01$ is the slowest because the estimation error of the surrogate increases as the trust region size increases according to Theorem \ref{thm:bound}.

\begin{figure}[ht]
\centering
\begin{subfigure}[b]{0.9\textwidth}
    \centering
    \includegraphics[width=\textwidth]{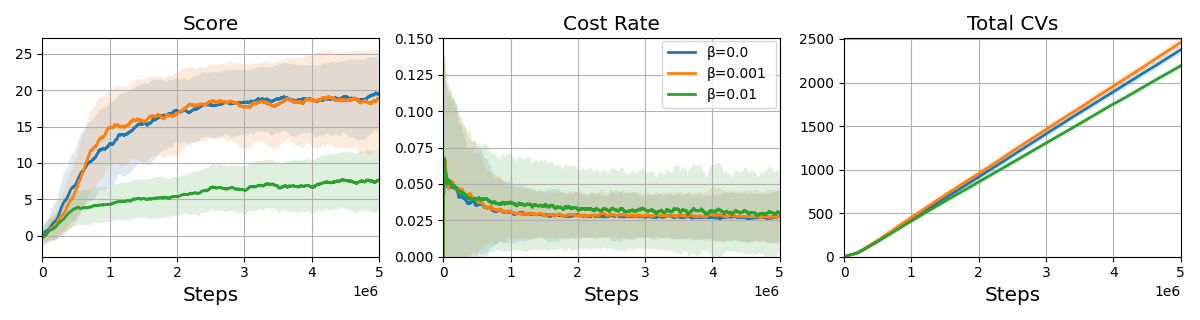}
    \caption{Entropy coefficient $\beta$.}
    \label{sfig:ablation on entropy}
\end{subfigure}
\begin{subfigure}[b]{0.9\textwidth}
    \centering
    \includegraphics[width=\textwidth]{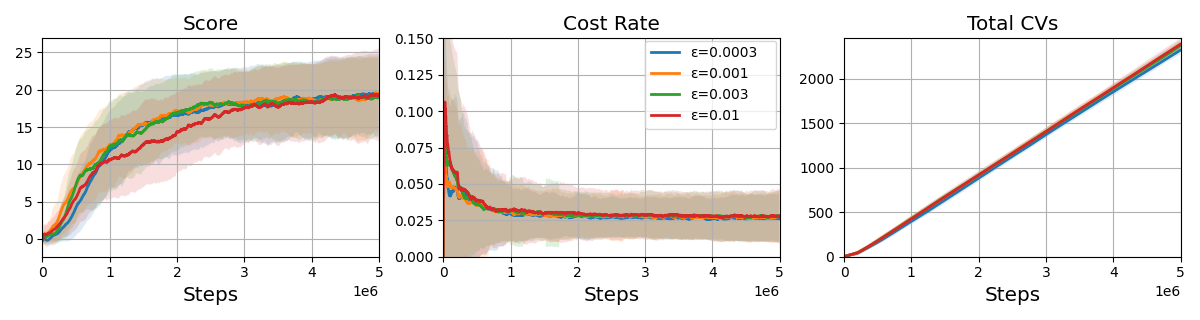}
    \caption{Trust region size $\epsilon$.}
    \label{sfig:ablation on trust region}
\end{subfigure}
\caption{Training curves of SDAC with different hyperparameters for the point goal task.}
\label{fig:ablation on trace decay}
\end{figure}

\newpage
\section{Comparison with RL Algorithms}
\label{sec:comparison_rl_algo}

In this section, we compare the proposed safe RL algorithm with traditional RL algorithms in the locomotion tasks and show that safe RL has the advantage of not requiring reward tuning.
We use the truncated quantile critic (TQC) \citep{kuznetsov2020tqc}, a state-of-the-art algorithm in existing RL benchmarks \citep{todorov2012mujoco}, as traditional RL baselines.
To apply the same experiment to traditional RL, it is necessary to design a reward reflecting safety. We construct the reward through a weighted sum as $\bar{R}=(R - \sum_{i=1}^3w_i C_i)/(1 + \sum_{i=1}^3 w_i)$, where $R$ and $C_{\{1,2,3\}}$ are used to train safe RL methods and are defined in Appendix \ref{sec:experimental settings}, and $R$ is called the \emph{true reward}.
The weights of the reward function $w_{\{1,2,3\}}$ are searched by a Bayesian optimization tool\footnote{We use Sweeps from Weights \& Biases \cite{wandb}.} to maximize the true reward of TQC in the Mini-Cheetah task.
Among the 63 weights searched through Bayesian optimization, the top five weights are listed in Table \ref{table: weight of reward}.


\begin{table}[ht]
\caption{Weights of the reward function for the Mini-Cheetah task.
}
\label{table: weight of reward}
\vskip 0.15in
\begin{center}
\begin{tabular}{l|l|l|l}
\toprule
Reward weights & $w_1$ & $w_2$ & $w_3$ \\
\midrule
$\#1$ & 1.588 & 0.299 & 0.174 \\
$\#2$ & 1.340 & 0.284 & 0.148 \\
$\#3$ & 1.841 & 0.545 & 0.951 \\
$\#4$ & 6.560 & 0.187 & 4.920 \\
$\#5$ & 1.603 & 0.448 & 0.564 \\
\bottomrule
\end{tabular}
\end{center}
\end{table}

Figure \ref{fig: training curves on locomotion} shows the training curves of the Mini-Cheetah task experiments where TQC is trained using the weight pairs listed in Table \ref{table: weight of reward}.
The graph shows that it is difficult for TQC to lower the second cost below the threshold while all costs of SDAC are below the threshold.
In particular, TQC with the fifth weight pairs shows the lowest second cost rate, but the true reward sum is the lowest.
This shows that it is challenging to obtain good task performance while satisfying the constraints through reward tuning.

\begin{figure}[ht]
\centering
\includegraphics[width=\linewidth]{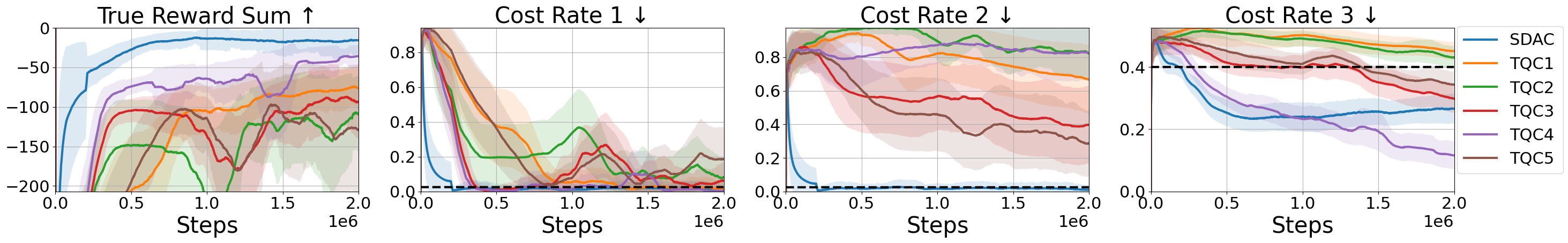}
\caption{
Training curves of the Mini-Cheetah task.
The black dashed lines show the thresholds used for the safe RL method.
The solid line represents the average value, and the shaded area shows one-fifth of the std value.
The number after TQC in the legend indicates which of the reward weights in Table \ref{table: weight of reward} is used.
All methods are trained with five different random seeds.
}
\label{fig: training curves on locomotion}
\end{figure}

\newpage
\section{Computational Cost Analysis}

\subsection{Complexity of Gradient Integration Method}
\label{sec: qualitative complexity}

In this section, we analyze the computational cost of the gradient integration method.
The proposed gradient integration method has three subparts.
First, it is required to calculate policy gradients of each cost surrogate, $g_k$, and $H^{-1}g_k$ for $\forall k \in \{1, 2,..., K\}$, where $H$ is the Hessian matrix of the KL divergence.
$H^{-1}g_k$ can be computed using the conjugate gradient method, which requires only a constant number of back-propagation on the cost surrogate, so the computational cost can be expressed as $K\cdot O(\mathrm{BackProp})$.

Second, the quadratic problem in Section \ref{sec:feasibility handling} is transformed to a dual problem, where the transformation process requires inner products between $g_k$ and $H^{-1}g_m$ for $\forall k,m \in \{1, 2,..., K\}$.
The computational cost can be expressed as $K^2\cdot O(\mathrm{InnerProd})$.

Finally, the transformed quadratic problem is solved in the dual space $\in \mathbb{R}^{K}$ using a quadratic programming solver.
Since $K$ is usually much smaller than the number of policy parameters, the computational cost almost negligible compared to the others.
Then, the cost of the gradient integration is $K\cdot O(\mathrm{BackProp}) + K^2\cdot O(\mathrm{InnerProd}) + C$.
Since the back-propagation and the inner products is proportional to the number of policy parameters $|\psi|$, the computational cost can be simplified as $O(K^2\cdot|\psi|)$. 

\subsection{Quantitative Analysis}
\label{sec: quantitative complexity}

\begin{table}[ht]
\caption{Training time of Safe RL algorithms (in hours). 
The training time of each algorithm is measured as the average time required for training with five random seeds.
The total training steps are $5\cdot 10^6$ and $3 \cdot 10^6$ for the point goal task and the Mini-Cheetah task, respectively.
}
\label{table: training time}
\vskip 0.15in
\begin{center}
\begin{tabular}{l|l|l|l|l|l}
\toprule
Task & \textbf{SDAC (proposed)} & OffTRC & WCSAC & CPO & CVPO \\
\midrule
Point goal (Safety Gym) & 7.96 & 4.86 & 19.07 & 2.61 & 47.43 \\
Mini-Cheetah (Locomotion) & 8.36 & 6.54 & 16.41 & 1.99 & - \\
\bottomrule
\end{tabular}
\end{center}
\end{table}

We analyze the computational cost of the proposed method quantitatively.
To do this, we measure the training time of the proposed method, SDAC, and the safe RL baselines.
We use a workstation whose CPU is the Intel Xeon e5-2650 v3, and GPU is the NVIDIA GeForce GTX TITAN X.
The results are presented in Table \ref{table: training time}.
While CPO is the fastest algorithm, its performance, such as the sum of rewards, is relatively poor compared to other algorithms.
The main reason why CPO shows the fastest computation time is that CPO is an on-policy algorithm, hence, it does not require an insertion to (and deletion from) a replay memory, and batch sampling.
SDAC shows the third fastest computation time in all algorithms and the second best one among off-policy algorithms.
Especially, SDAC is slightly slower than OffTRC, which is the fastest one among off-policy algorithms.
This result shows the benefit of SDAC since SDAC outperforms OffTRC in terms of the returns and CV, but the training time is not significantly increased over OffTRC.
WCSAC, which is based on SAC, has a slower training time because it updates networks more frequently than other algorithms.
CVPO, an EM-based safe RL algorithm, has the slowest training time.
In the E-step of CVPO, a non-parametric policy is optimized to solve a local subproblem, and the optimization process requires discretizing the action space and solving a non-linear convex optimization for all batch states.
Because of this, CVPO takes the longest to train an RL agent.



\end{document}